\documentclass[11pt]{article} 

\usepackage{url}
\usepackage{hyperref}
\usepackage[left=1in,right=1in,top=1in,bottom=1in]{geometry}
\usepackage{authblk}
\usepackage[usenames, dvipsnames, table]{xcolor}
\usepackage{caption}
\usepackage{tabularx}
\usepackage{booktabs}
\usepackage{array}
\usepackage{multirow}
\usepackage{cellspace}
\usepackage{colortbl}
\usepackage{paralist}
\usepackage{verbatim}
\usepackage{algorithmic, algorithm}
\usepackage{mdwmath}
\usepackage{amssymb}
\usepackage{amsmath}
\usepackage{amsthm}
\usepackage{xfrac}
\usepackage{mathtools}  
\usepackage{relsize}
\usepackage{xspace}
\usepackage{tabulary}

\usepackage{epsfig}
\usepackage{wrapfig}
\usepackage{lipsum}
\usepackage{subfigure}

\usepackage{flushend}
\usepackage{xfrac}
\usepackage{nicefrac}
\usepackage{enumerate}
\usepackage{paralist}
\usepackage{enumitem}

\usepackage{pgf}
\usepackage{tikz}
\usepackage{pgfplots}
\usepackage[utf8]{inputenc}
\usetikzlibrary{positioning}
\usetikzlibrary{shapes, arrows, decorations.markings}
\usetikzlibrary{calc}
\usetikzlibrary{fit}

\usepackage[square,sort,comma,numbers]{natbib}

\newtheorem{theorem}{Theorem}

\newtheorem{lemma}{Lemma}

\newtheorem{remark}{Remark}

\newtheorem{corollary}{Corollary}

\makeatletter
\newtheorem*{rep@theorem}{\rep@title}
\newcommand{\newreptheorem}[2]{%
\newenvironment{rep#1}[1]{%
 \def\rep@title{#2 \ref{##1}}%
 \begin{rep@theorem}}%
 {\end{rep@theorem}}}
\makeatother
\newreptheorem{theorem}{Theorem}
\newreptheorem{lemma}{Lemma}

\DeclareMathOperator*{\argmax}{arg\,max}

\makeatletter
\let\save@mathaccent\mathaccent
\newcommand*\if@single[3]{%
  \setbox0\hbox{${\mathaccent"0362{#1}}^H$}%
  \setbox2\hbox{${\mathaccent"0362{\kern0pt#1}}^H$}%
  \ifdim\ht0=\ht2 #3\else #2\fi
  }
\newcommand*\rel@kern[1]{\kern#1\dimexpr\macc@kerna}
\newcommand*\widebar[1]{\@ifnextchar^{{\wide@bar{#1}{0}}}{\wide@bar{#1}{1}}}
\newcommand*\wide@bar[2]{\if@single{#1}{\wide@bar@{#1}{#2}{1}}{\wide@bar@{#1}{#2}{2}}}
\newcommand*\wide@bar@[3]{%
  \begingroup
  \def\mathaccent##1##2{%
    \let\mathaccent\save@mathaccent
    \if#32 \let\macc@nucleus\first@char \fi
    \setbox\z@\hbox{$\macc@style{\macc@nucleus}_{}$}%
    \setbox\tw@\hbox{$\macc@style{\macc@nucleus}{}_{}$}%
    \dimen@\wd\tw@
    \advance\dimen@-\wd\z@
    \divide\dimen@ 3
    \@tempdima\wd\tw@
    \advance\@tempdima-\scriptspace
    \divide\@tempdima 10
    \advance\dimen@-\@tempdima
    \ifdim\dimen@>\z@ \dimen@0pt\fi
    \rel@kern{0.6}\kern-\dimen@
    \if#31
      \overline{\rel@kern{-0.6}\kern\dimen@\macc@nucleus\rel@kern{0.4}\kern\dimen@}%
      \advance\dimen@0.4\dimexpr\macc@kerna
      \let\final@kern#2%
      \ifdim\dimen@<\z@ \let\final@kern1\fi
      \if\final@kern1 \kern-\dimen@\fi
    \else
      \overline{\rel@kern{-0.6}\kern\dimen@#1}%
    \fi
  }%
  \macc@depth\@ne
  \let\math@bgroup\@empty \let\math@egroup\macc@set@skewchar
  \mathsurround\z@ \frozen@everymath{\mathgroup\macc@group\relax}%
  \macc@set@skewchar\relax
  \let\mathaccentV\macc@nested@a
  \if#31
    \macc@nested@a\relax111{#1}%
  \else
    \def\gobble@till@marker##1\endmarker{}%
    \futurelet\first@char\gobble@till@marker#1\endmarker
    \ifcat\noexpand\first@char A\else
      \def\first@char{}%
    \fi
    \macc@nested@a\relax111{\first@char}%
  \fi
  \endgroup
}
\makeatother

\newcommand{\bottomrulec}{%
  \arrayrulecolor{gray!10}\specialrule{\belowrulesep}{0pt}{0pt}
  \arrayrulecolor{black}\specialrule{\heavyrulewidth}{0pt}{0pt}
  \arrayrulecolor{black}
}
\newcommand{\toprulec}{%
  \arrayrulecolor{black}\specialrule{\heavyrulewidth}{\aboverulesep}{0pt}
  \arrayrulecolor{white}\specialrule{\belowrulesep}{0pt}{0pt}
  \arrayrulecolor{black}
}
\newcommand{\midrulec}{%
  \arrayrulecolor{white}\specialrule{\aboverulesep}{0pt}{0pt}
  \arrayrulecolor{black}\specialrule{\lightrulewidth}{0pt}{\belowrulesep}
}

\newcommand{\opt}{\mathsf{\normalfont{OPT}}}

\newcommand{\eqdef}{{\triangleq}}

\makeatletter

\def\makecalletters#1{%
\expandafter\newcommand\csname cal#1\endcsname{\mathcal{#1}}}
\def\makehatbfletters#1{%
\expandafter\newcommand\csname #1hat\endcsname{\widehat{\mathbf{#1}}}}
\def\maketildebfletters#1{%
\expandafter\newcommand\csname #1tilde\endcsname{\widetilde{\mathbf{#1}}}}

\count@=0
\loop
\advance\count@ 1
\edef\y{\@Alph\count@}%
\expandafter\makecalletters\y
\expandafter\makehatbfletters\y
\expandafter\maketildebfletters\y
\ifnum\count@<26
\repeat

\newcommand{\Tsvd}{T_{\texttt{SVD}}}
\newcommand{\Tsketch}{T_{\texttt{SKETCH}}}
\newcommand{\X}{\mathbf{X}}
\newcommand{\Y}{\mathbf{X}}
\newcommand{\A}{\mathbf{A}}

\newcommand{\Xsharp}{\X_{\sharp}}

\newcommand{\Xstar}{\X_{\star}}

\newcommand{\Xset}{\mathcal{X}}

\def\x{\mathbf{x}}

\def\R{\mathbb{R}}
\def\Z{\mathbb{Z}}

\def\L{\mathbf{\Lambda}}
\def\C{\mathcal{C}}
\def\bC{\mathbf{C}}
\def\N{\mathcal{N}}
\def\W{\mathbf{W}}

\def\A{\mathbf{A}}
\def\Y{\mathbf{Y}}

\def\y{\mathbf{y}}

\def\U{\mathbf{U}}

\newcommand{\dimension}{d}

\newcommand{\transpose}{{\top}} 
\newcommand{\T}{{\transpose}} 
\newcommand{\frob}{{\textnormal{F}}}
\newcommand{\supp}{\text{supp}}

\newcommand{\Tr}{\text{\textsc{Tr}}\mathopen{}}
\newcommand{\trace}{\Tr}

\newcommand{\matV}{{\bf V}}

\newcommand{\OPT}{{\normalfont{\textsf{{\mbox{OPT}}}}}\xspace}

\newcommand{\range}[2][{1}]{\lbrace{#1,\hdots,#2}\rbrace\xspace}

\title{\huge Sparse PCA via Bipartite Matchings}

\author{
Megasthenis~Asteris$^{\alpha}$, Dimitris Papailiopoulos$^{\beta}$, Anastasios~Kyrillidis$^{\alpha}$
 Alexandros~G.~Dimakis$^{\alpha}$\\
$^{\alpha}$UT Austin, $^{\beta}$UC Berkeley
} 

\date{August 2015}

\newcommand{\PTAS}{\textsf{PTAS}}

\begin{document}

\maketitle

\begin{abstract}
We consider the following multi-component sparse PCA problem:
given a set of data points, we seek to extract a small number of sparse components with \emph{disjoint} supports that jointly capture the maximum possible variance.
These components can be computed one by one, repeatedly solving the single-component problem and deflating the input data matrix, but as we show this greedy procedure is suboptimal.
We present a novel algorithm for sparse PCA that jointly optimizes multiple disjoint components. 
The extracted features capture variance that lies within a multiplicative factor arbitrarily close to $1$ from the optimal.
Our algorithm is combinatorial and computes the desired components by solving multiple instances of the bipartite maximum weight matching problem.
Its complexity grows as a low order polynomial in the ambient dimension of the input data matrix, but exponentially in its rank.
However, it can be effectively applied on a low-dimensional sketch of the data; this allows us to obtain polynomial-time approximation guarantees via spectral bounds.
We evaluate our algorithm on real data-sets and empirically demonstrate that 
in many cases it outperforms existing, deflation-based approaches.
\end{abstract}

\section{Introduction}

Principal Component Analysis (PCA) reduces the dimensionality of a data set by projecting it onto principal subspaces spanned by the leading eigenvectors of the sample covariance matrix. 
Sparse PCA is a useful variant that offers higher data interpretability~\cite{kaiser1958varimax,jolliffe1995rotation,jolliffe2003modified}, a 
property that is sometimes desired even at the cost of statistical fidelity~\cite{zou2006sparse}. Furthermore, when 
the obtained features are used in subsequent learning tasks, sparsity potentially leads to better generalization error~\cite{boutsidis2011sparse}.

Given a real ${{n} \times \dimension}$ data matrix~$\mathbf{S}$ representing~${n}$ 
centered data points supported on $\dimension$ features, 
the leading sparse principal component of the data set is the sparse vector that maximizes the explained variance:
\begin{align}
	\x_\star\;
	\eqdef
	\argmax_{
		\substack{
			\|\x\|_2 = 1, \|\x\|_0 = {s}}
	} \x^{\T} \A \x,
\label{spca-single-def}
\end{align}
where $\mathbf{A} = \sfrac{1}{{n}} \cdot \mathbf{S}^{\T}\mathbf{S}$ is the $\dimension \times \dimension$ empirical covariance matrix.
The sparsity constraint makes the problem NP-hard and hence computationally intractable in general, and hard to approximate within some small constant \cite{siuon2015on}. 
A significant volume of prior work has focused on algorithms that approximately solve the optimization problem 
\cite{jolliffe1995rotation,
jolliffe2003modified,
zou2006sparse,
journee2010generalized,
moghaddam2006spectral,
d2008optimal, zhang2012sparse, d2007direct,
yuan2013truncated,
sigg:2008,
papailiopoulos2013sparse,
spanspca:IT2014},
while a large volume of theoretical results  has been established under planted statistical models
\cite{amini2008high,
ma2013sparse,
d2012approximation,
cai2012sparse,
deshpande2013sparse,
berthet2013optimal,
berthet2013complexity,
WBS14-more_SPCA_planted-hardness,
KNV15-SDP-Sparse_PCA}.

In most practical settings, we tend to go beyond computing a single sparse PC.
Contrary to the single-component problem, there has been limited work on computing multiple components.
The scarcity is partially attributed to conventional PCA wisdom:
multiple components can be computed one-by-one, repeatedly, by solving the single-component sparse PCA problem~\eqref{spca-single-def} and \emph{deflating} the input data to remove information captured by previously extracted components~\cite{mackey2009deflation}. 
In fact, the multi-component version of sparse PCA is not uniquely defined in the literature.
Different deflation-based approaches can lead to different outputs: extracted components may or may not be orthogonal, while they may have disjoint or overlapping supports~\cite{mackey2009deflation}.
In the statistics literature, 
where the objective is typically to recover a ``true" principal subspace,
a branch of work has focused on the ``subspace row sparsity"~\cite{vu2012minimax},
an assumption that leads to sparse components all supported on the same set of variables. 
While in~\cite{magdon:linearencoders}, the authors discuss an alternative perspective on the fundamental objective of the sparse PCA problem.

In this work, 
we develop a novel algorithm for the multi-component sparse PCA problem with disjoint supports. 
Formally, we are interested in finding $k$ components that are $s$-sparse, have disjoint supports,
and jointly maximize the explained variance:	
\begin{align}
	\Xstar
	\;\eqdef\;
	\argmax_{\X \in \Xset_{k}} 
	\trace\bigl(\X^\top \A \X \bigr),
	\label{spca-multi-def}
\end{align}
where the feasible set is
\begin{align}
	\Xset_{k} 
	\;\eqdef\;
	\bigl\lbrace 
		\X \in \mathbb{R}^{\dimension \times {k}}: 
		\|\X^j\|_{2} = 1, 
		\|\X^j\|_{0} = {s},
		{\supp(\X^i) \cap \supp(\X^j) = \emptyset},
		\;\forall\,j \in [k], i < j
	\bigr\rbrace,
	\nonumber
\end{align}
with $\X^{j}$ denoting the $j$th column of $\X$.
The number $k$ of the desired components is a user defined parameter and we consider it to be a small constant.

Contrary to the greedy sequential approach that repeatedly uses deflation, our algorithm  \textit{jointly} computes all the vectors in $\X$,
and comes with theoretical approximation guarantees.
We note that even if one could solve each single-component sparse PCA problem~\eqref{spca-single-def} {\it exactly}, greedy deflation can be highly suboptimal. 
We show this through a simple example in Section~\ref{sec:deflation-suboptimality}.

\paragraph{Our Contributions}
\begin{enumerate}[leftmargin=1.5em]
\item
We develop an algorithm that provably approximates the solution to the sparse PCA problem~\eqref{spca-multi-def} within a multiplicative factor arbitrarily close to $1$.
To the best of our knowledge, this is the first algorithm that jointly optimizes multiple components with disjoint supports, provably.
Our algorithm is combinatorial; it recasts sparse PCA as multiple instances of {\it bipartite maximum weight matching} on graphs determined by the input data.
\item
	The computational complexity of our algorithm grows as a low order polynomial in the ambient dimension $\dimension$,
	but is exponential in the intrinsic dimension of the input data, 
	\textit{i.e.}, the rank of~$\A$. 
To alleviate the impact of this dependence, 
 our algorithm can be applied on a low-dimensional sketch of the input data to obtain an approximate solution to~\eqref{spca-multi-def}.
 This extra level of approximation introduces an additional penalty in our theoretical approximation guarantees,
 which naturally depends on the quality of the sketch and, in turn, the spectral decay of~$\mathbf{A}$.
We show how these bounds further translate to an additive \PTAS{} (polynomial-time approximation scheme) for sparse PCA.
Our additive PTAS outputs an approximate solution with explained variance of at least $\OPT-\epsilon 	\cdot s$, for any sparsity $s \in\{1,\ldots, n\}$, any constant error $\epsilon>0$ and any $k = O(1)$ number of orthogonal components.\footnote{Here, $\OPT$ is the explained variance captured by the optimal set of $k$ components that are $s$ sparse and have disjoint supports.}
\item
	We empirically evaluate our algorithm on real datasets,
	and compare it against state-of-the-art methods for the single-component sparse PCA problem~\eqref{spca-single-def} in conjunction with the appropriate deflation step.
	In many cases, our algorithm---as a result of jointly optimizing over multiple components---leads to significantly improved results,  and outperforms deflation-based approaches. 
\end{enumerate}

\section{Sparse PCA through Bipartite Matchings}
Our algorithm approximately solves the constrained maximization~\eqref{spca-multi-def} on a $\dimension \times \dimension$ rank-$r$ PSD matrix~${\A}$ within a multiplicative factor arbitrarily close to $1$.
It operates by recasting the maximization into multiple instances of the bipartite maximum weight matching problem.
Each instance ultimately yields a feasible solution:
a set of $k$ components that are $s$-sparse and have disjoint supports.
The algorithm examines these solutions, and outputs the one that maximizes the explained variance, \textit{i.e.}, the quadratic objective in~\eqref{spca-multi-def}.

The computational complexity of our algorithm grows as a low order polynomial in the ambient dimension~$d$ of the input, but exponentially in its rank~$r$.
Despite the unfavorable dependence on the rank, it is unlikely that a substantial improvement can be achieved in general~\cite{siuon2015on}.
However, decoupling the dependence on the ambient and the intrinsic dimension of the input has an interesting ramification; 
instead of the original input ${\A}$, 
our algorithm can be applied on a low-rank surrogate to obtain an approximate solution, alleviating the dependence on~$r$.
We discuss this in Section~\ref{sec:spca-on-sketch}, and present the approximation bound that this allows us to obtain.

Let ${\A} = {\mathbf{U}}{\mathbf{\Lambda}}{\mathbf{U}}^{\transpose}$
denote the truncated eigenvalue decomposition of ${\A}$;
${\mathbf{\Lambda}}$ is a diagonal $r \times r$ whose $i$th diagonal entry is equal to the $i$th largest eigenvalue of ${\A}$,
while the columns of ${\mathbf{U}}$ are the corresponding eigenvectors. 
By the Cauchy-Schwartz inequality, 
for any $\mathbf{x} \in \mathbb{R}^{\dimension}$,  
   \begin{align}
	  \mathbf{x}^{\transpose}{\A}\mathbf{x}
	  =
	  \bigl\|
	  	{\mathbf{\Lambda}}^{1/2}{\mathbf{U}}^{\transpose} \mathbf{x}
	  \bigr\|_{2}^{2}
	  \ge
	  \bigl\langle
		 {\mathbf{\Lambda}}^{1/2}{\mathbf{U}}^{\transpose} \mathbf{x},\,
		 \mathbf{c}
	  \bigr\rangle^{2},
	  \quad
	  \forall\; 
	  \mathbf{c} \in \mathbb{R}^{r}: \|\mathbf{c}\|_{2}=1.
	  \label{cauchy-single}
   \end{align}
   In fact, equality in~\eqref{cauchy-single} can always be achieved for $\mathbf{c}$ colinear to ${\mathbf{\Lambda}}^{1/2}{\mathbf{U}}\mathbf{x} \in \mathbb{R}^{r}$ and in turn
   \begin{align}
   		\mathbf{x}^{\transpose}{\A}\mathbf{x}
	  	=
		\max_{\mathbf{c} \in {\mathbb{S}_{2}^{r-1}}}
		\bigl\langle
		 \mathbf{x},\,
		 {\mathbf{U}}{\mathbf{\Lambda}}^{1/2}\mathbf{c}
	  \bigr\rangle^{2},
	  \nonumber
   \end{align}
   where $\mathbb{S}_{2}^{r-1}$ denotes the $\ell_{2}$-unit sphere in $r$ dimensions. 
   More generally, for any $\X \in \mathbb{R}^{\dimension \times {k}}$,
      \begin{align}
   \trace\left(
		\X^{\transpose}{\A} \X
   \right)
   =
   \sum_{j=1}^{{k}}
   	{\X^{j}}^{\transpose}{\A} \X^{j}
	=
		\max_{\mathbf{C}:{\mathbf{C}^{j}} \in {{\mathbb{S}_{2}^{r-1}}} \forall j}
	\sum_{j=1}^{{k}}
	\bigl\langle
		\X^{j},\,
		{\mathbf{U}}{\mathbf{\Lambda}}^{1/2}\mathbf{C}^{j}
	\bigr\rangle^{2}.
	\label{mt-trace-alt-def}
	\end{align}
Under the variational characterization of the trace objective in~\eqref{mt-trace-alt-def},
the sparse PCA problem~\eqref{spca-multi-def} can be re-written as a joint maximization over the variables $\X$ and $\mathbf{C}$ as follows:
\begin{align}
	\max_{\X \in \Xset_{k}} 
	\trace\bigl(\X^\top \A \X \bigr)
	\;=\;
	\max_{\X \in \mathcal{X}_{k}}
	\max_{\mathbf{C}:{\mathbf{C}^{j}} \in {{\mathbb{S}_{2}^{r-1}}} \forall j}
	\sum_{j=1}^{{k}}
	\bigl\langle
		\X^{j},\,
		{\mathbf{U}} {\mathbf{\Lambda}}^{1/2}\mathbf{C}^{j}
	\bigr\rangle^{2}.
	\label{double-maximization}
\end{align}
The alternative formulation of the sparse PCA problem in~\eqref{double-maximization}
takes a step towards decoupling the dependence of the optimization on the ambient and intrinsic dimensions $\dimension$ and ${r}$, respectively.
The motivation behind the introduction of the auxiliary variable~$\mathbf{C}$ will become clear in the sequel.

For a given~$\mathbf{C}$,
the value of~$\X \in \Xset_{k}$ that maximizes the objective in~\eqref{double-maximization} for that $\mathbf{C}$ is
\begin{align}
	\widehat{\X}
	\;\eqdef\; 
	\argmax_{
		\substack{\X \in \Xset_{{k}}}
	}
	\sum_{j=1}^{{k}}
		\left\langle \X^j, \W^j \right\rangle^2,
	\label{local-problem}
\end{align}
where $\W \eqdef {\U} {\L}^{1/2} \mathbf{C}$ is a real $\dimension \times {k}$ matrix.
The constrained, non-convex maximization~\eqref{local-problem} plays a central role in our developments.
We will later describe a combinatorial ${O\mathopen{}(\dimension \cdot ({s}\cdot {k})^{2})}$ procedure to efficiently compute \scalebox{0.90}{$\widehat{\X}$},
reducing the maximization to an instance of the bipartite maximum weight matching problem.
For now, however, let us assume that such a procedure exists.

Let ${\X}_{\star}$, ${\mathbf{C}}_{\star}$ be the pair that attains the maximum in~\eqref{double-maximization};
in other words,~${\X}_{\star}$ is the desired solution to the sparse PCA problem.
If the optimal auxiliary variable ${\mathbf{C}}_{\star}$ was known,
then we would be able to recover~${\X}_{\star}$ by solving the maximization~\eqref{local-problem} for~$\mathbf{C}={\mathbf{C}}_{\star}$.
Of course, ${\mathbf{C}}_{\star}$ is not known, and it is not possible to exhaustively consider all possible values in the domain of $\mathbf{C}$.
Instead, we examine only a finite number of possible values of $\mathbf{C}$ over a  fine discretization of its domain.
In particular, let $\mathcal{N}_{\mathsmaller{\sfrac{\epsilon}{2}}}(\scalebox{0.85}{$\mathbb{S}_{2}^{r-1}$})$ denote a finite \mbox{$\sfrac{\epsilon}{2}$-net} of the $r$-dimensional $\ell_{2}$-unit sphere;
for any point in $\mathbb{S}_{2}^{r-1}$, the net contains a point within an $\sfrac{\epsilon}{2}$ radius from the former. 
There are several ways to construct such a net~\cite{matouvsek2002lectures}.
Further, let 
$[\mathcal{N}_{\mathsmaller{\sfrac{\epsilon}{2}}}(\scalebox{0.85}{$\mathbb{S}_{2}^{r-1}$})]^{\otimes {k}} \subset \mathbb{R}^{\dimension \times k}$ 
denote the $k$th Cartesian power of the aforementioned $\sfrac{\epsilon}{2}$-net.
By construction, this collection of points contains a matrix $\mathbf{C}$ that is column-wise close to ${\mathbf{C}}_{\star}$.
In turn, 
it can be shown using the properties of the net, 
that the candidate solution $\X \in \Xset_{k}$ obtained through~\eqref{local-problem} at that point $\mathbf{C}$ will be approximately as good as the optimal ${\X}_{\star}$ in terms of the quadratic objective in~\eqref{spca-multi-def}.

\setlength{\columnsep}{1.5em}
\begin{wrapfigure}[14]{R}[0pt]{0.6\textwidth}
\begin{minipage}{0.6\textwidth}
\vspace{-1.0em}
\small
\begin{algorithm}[H] 
   \caption{Sparse PCA (Multiple disjoint components)}
   \label{algo:lowrank-symmetric-multi-component}
   \begin{algorithmic}[1]
   		\INPUT: PSD $\dimension \times \dimension$ rank-$r$ matrix ${\A}$, ${\epsilon \in (0,1)}$, ${k} \in \Z_{+}$.
		\OUTPUT: $\widebar{\X} \in \Xset_{k}$
		\hfill\COMMENT{Theorem~\ref{thm:spca-multi-rank-r-guarantees}}
   		\STATE $\C \gets \lbrace \rbrace$ 
   		\STATE $\left[\U, \L \right] \leftarrow \texttt{EIG}({\A})$
   		\FOR {\textbf{each} $\mathbf{C} \in [\N_{\mathsmaller{\sfrac{\epsilon}{2}}}(\scalebox{0.85}{$\mathbb{S}_{2}^{r-1}$})]^{\otimes {k}}$}
      		\STATE $\W \gets \U\L^{1/2}\mathbf{C}$  
		\hfill \COMMENT{$\W \in \R^{\dimension \times {k}}$}
      		\STATE 
			\scalebox{0.95}{
				$\widehat{\X} \gets \argmax_{\substack{\X \in \Xset_{k} }} \sum_{j=1}^{{k}} \left\langle \X^j, \W^j \right\rangle^2$%
			}
		\hfill \COMMENT{Alg.~\ref{algo:local-candidate}}
      		\STATE $\C \gets \C \cup \bigl\lbrace{\widehat{\X}}\bigr\rbrace$
   		\ENDFOR
		\STATE $\widebar{\X} \gets \argmax_{\substack{ \X \in \C } } \trace\left(\X^{\T} {\A} \X \right)$
   \end{algorithmic}
\end{algorithm}
\end{minipage}
\end{wrapfigure}
All above observations yield a procedure for approximately solving the sparse PCA problem~\eqref{spca-multi-def}.
The steps are outlined in Algorithm~\ref{algo:lowrank-symmetric-multi-component}.
Given the desired number of components $k$ and an accuracy parameter $\epsilon\in (0,1)$,
the algorithm generates a net $[\mathcal{N}_{\mathsmaller{\sfrac{\epsilon}{2}}}(\scalebox{0.85}{$\mathbb{S}_{2}^{r-1}$})]^{\otimes {k}}$ and iterates over its points. 
At each point $\mathbf{C}$,
it computes a feasible solution for the sparse PCA problem --
a set of $k$ $s$-sparse components --
by solving the maximization in~\eqref{local-problem} via a procedure (Alg.~\ref{algo:local-candidate}) that will be described in the sequel.
The algorithm collects the candidate solutions identified at the points of the net.
The best among them achieves an objective in~\eqref{spca-multi-def} that provably lies close to optimal. 
More formally,
\begin{theorem}
   \label{thm:spca-multi-rank-r-guarantees}
   For any  real $\dimension \times \dimension$ rank-$r$ PSD matrix ${\A}$,
   desired number of components ${k}$,
   number ${s}$ of nonzero entries per component, 
   and accuracy parameter $\epsilon \in (0,1)$,
   Algorithm~\ref{algo:lowrank-symmetric-multi-component}
   outputs $\widebar{\X} \in \Xset_{k}$ such that
   \begin{align}
   	\trace\bigl(
		\widebar{\X}^{\transpose} {\A} \widebar{\X}
	 \bigr)
	  \;\ge\;
	  \left(1-\epsilon \right) \cdot
	  \trace\bigl(
	  	{\X}_{\star}^{\transpose} {\A} {\X}_{\star}
	  \bigr),
	  \nonumber
   \end{align}
   where
      $
   {\X}_{\star} \eqdef \argmax_{\X \in \Xset_{k}}
   	\trace\left(\X^{\transpose}{{\A}}\X\right),
   $
   in time 
   $\Tsvd(r) + O\mathopen{}\bigl(
   		\bigl(\tfrac{4}{\epsilon}\bigr)^{r \cdot {{k}}} \cdot \dimension \cdot ({s} \cdot {k})^{2} \bigr)$.
\end{theorem}
Algorithm~\ref{algo:lowrank-symmetric-multi-component} is the first nontrivial algorithm that provably approximates the solution of the sparse PCA problem~\eqref{spca-multi-def}.
According to Theorem~\ref{thm:spca-multi-rank-r-guarantees},
it achieves an objective value that lies within a multiplicative factor from the optimal, arbitrarily close to~$1$.
Its complexity grows as a low-order polynomial in the dimension $d$ of the input, but exponentially in the intrinsic dimension $r$.
Note, however, that it can be exponentially faster compared to the ${O(d^{s\cdot k})}$ brute force approach that exhaustively considers all candidate supports for the $k$ sparse components. 
The complexity of our algorithm follows from the cardinality of the net
and the complexity of Algorithm~\ref{algo:local-candidate},
the subroutine that solves the constrained maximization~\eqref{local-problem}. 
The latter is a key ingredient of our algorithm, and is discussed in detail in the next subsection.
A formal 
proof of Theorem~\ref{thm:spca-multi-rank-r-guarantees} is provided in Section~\ref{sec:proof-algo-rank-r-spca}.

\subsection{Sparse Components via Bipartite Matchings}
\vspace{-0.2em}
In the core of Algorithm~\ref{algo:lowrank-symmetric-multi-component} lies 
Algorithm~\ref{algo:local-candidate}, a procedure that solves the constrained maximization in~\eqref{local-problem}.
The algorithm breaks down the maximization into two stages.
First, it identifies the support of the optimal solution \scalebox{0.9}{$\widehat{\X}$}.
Determining the support reduces to an instance of the maximum matching problem on a weighted bipartite graph $G$. 
Then, it recovers the exact values of the nonzero entries in \scalebox{0.92}{$\widehat{\X}$} based on the Cauchy-Schwarz inequality. 
In the sequel, we provide a brief description of Algorithm~\ref{algo:local-candidate}, leading up to its guarantees in Lemma~\ref{lemma:algo-local-candidate-guarantees}.

Let $\mathcal{I}_j \eqdef \supp(\widehat{\X}^j)$
be the support of the $j$th column of $\widehat{\X}$, ${j = 1,\hdots, {k}}$. 
The objective in \eqref{local-problem} becomes
\begin{align}
	\sum_{j=1}^{{k}}
		\bigl\langle 
			\widehat{\X}^j, \W^j 
		\bigr\rangle^{2} 
	=
	\sum_{j=1}^{{k}}
	\Bigl( 
		\sum_{i \in \mathcal{I}_j} 
		\widehat{X}_{ij} \cdot W_{ij} 
	\Bigr)^{2} 
	\le
	\sum_{j=1}^{{k}} 
		\sum_{i \in \mathcal{I}_{j}}W_{ij}^{2}.
	\label{local-objective-upper-bound}
\end{align} 
The last inequality is an application of the Cauchy-Schwarz Inequality and the constraint 
${\|\X^j \|_{2} = 1}$ $\forall\, j \in \range{{k}}$.
In fact, if an oracle reveals the supports $\mathcal{I}_j$, ${j=1,\hdots,{k}}$, the upper bound in~\eqref{local-objective-upper-bound} can always be achieved by setting the nonzero entries of \scalebox{0.9}{$\widehat{\X}$} as in Algorithm~\ref{algo:local-candidate} (Line $6$). 
Therefore, 
the key in solving~\eqref{local-problem} is determining the collection of supports
to maximize the right-hand side of~\eqref{local-objective-upper-bound}.

\setlength{\columnsep}{1.5em}
\begin{wrapfigure}[16]{R}[0pt]{0.42\textwidth}
	\vspace{-0.9em}
	\centering
	   \includegraphics[width=0.38\textwidth]{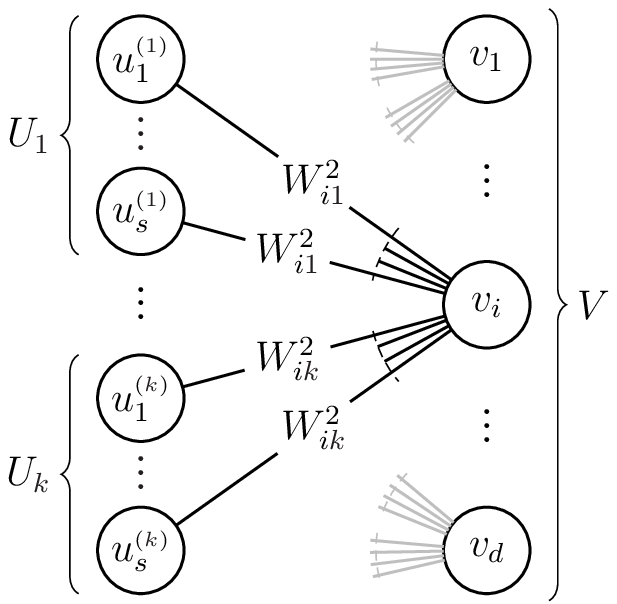} 
	   \caption{
	   	The graph~$G$ generated by Alg.~\ref{algo:local-candidate}.
		It is used to determine the support of the solution 
		\protect\scalebox{0.9}{$\widehat{\X}$} in~\eqref{local-problem}. 
	}
	\label{fig:spca-bipartite}
\end{wrapfigure}
By constraint,
the sets $\mathcal{I}_{j}$ must be pairwise disjoint, each with cardinality~${s}$.
Consider a weighted bipartite graph 
${G=\bigl({U = \lbrace U_{1}, \hdots, U_{k}\rbrace},V,E \bigr)}$ 
constructed as follows%
\footnote{%
The construction is formally outlined in Algorithm~\ref{algo:graph-gen} in Section~\ref{sec:gen-bipart-graph}.
}
 (Fig.~\ref{fig:spca-bipartite}):
\begin{itemize}[leftmargin=*]
	\item $V$ is a set of $\dimension$ vertices $v_{1}, \hdots, v_{\dimension}$, corresponding to the $\dimension$ variables, 
	\textit{i.e.}, the $\dimension$ rows of \scalebox{0.9}{$\widehat{\X}$}.
	\item $U$ is a set of ${k}\cdot {s}$ vertices, conceptually partitioned into ${k}$ disjoint subsets $U_{1}, \hdots, U_{{k}}$, each of cardinality~${s}$.
	The $j$th subset,~$U_{j}$, is associated with the support~$\mathcal{I}_{j}$;
	the~${s}$ vertices $u^{\mathsmaller{(j)}}_{\alpha}$, ${\alpha=1,\hdots,s}$ in~$U_{j}$ serve as placeholders for the variables/indices in~$\mathcal{I}_{j}$.
	\item Finally, the edge set is ${E = U \times V}$. 
	The edge weights are determined by the $\dimension \times {k}$ matrix $\W$ in~\eqref{local-problem}.
	In particular, the weight of edge $(u^{\mathsmaller{(j)}}_{\alpha}, v_{i})$ is equal to~$W_{ij}^{2}$.
	Note that all vertices in $U_{j}$ are effectively identical; they all share a common neighborhood and edge weights.
\end{itemize}

\begin{algorithm}[!ht]
	\caption{Compute Candidate Solution} 
	\label{algo:local-candidate}
	\begin{algorithmic}[1]
		\INPUT Real $\dimension \times {k}$ matrix $\W$
		\OUTPUT $ \widehat{\X}  = \argmax_{\X \in \Xset_{k}} \sum_{j=1}^{{k}} \bigl\langle \X^j, \W^j \bigr\rangle^{2}$
		\STATE $G\bigl(\lbrace{U_{j}}\rbrace_{j=1}^{{k}}, V, E\bigr) \gets \textsc{GenBiGraph}(\W)$
			\hfill \COMMENT{Alg.~\ref{algo:graph-gen}}
		\STATE $\mathcal{M} \gets \textsc{MaxWeightMatch}(G)$
			\hfill \COMMENT{$\subset E$}
		\STATE $\widehat{\X} \gets \mathbf{0}_{\dimension \times {k}}$
		\FOR{$j=1,\hdots, {k}$} 
			\STATE $\mathcal{I}_{j} \gets \left\lbrace i \in \range{\dimension}: (u, v_{i}) \in \mathcal{M}, u \in {U}_{j} \right\rbrace$ 
			\STATE $[\widehat{\X}^{j}]_{\mathcal{I}_{j}} \gets [\W^{j}]_{\mathcal{I}_{j}} / \|[\W^{j}]_{\mathcal{I}_{j}} \|_{2}$		\ENDFOR
	\end{algorithmic}
\end{algorithm}

Any feasible support 
\scalebox{0.9}{$\lbrace\mathcal{I}_{j}\rbrace_{j=1}^{k}$}
 corresponds to a \emph{perfect matching} in~$G$ and vice-versa.
Recall that a {matching} is a subset of the edges containing no two edges incident to the same vertex,
while a {perfect matching}, in the case of an unbalanced bipartite graph  ${G=(U,V,E)}$ with $|U| \le |V|$, is a matching that contains at least one incident edge for each vertex in $U$. 
Given a perfect matching $\mathcal{M} \subseteq E$,
the disjoint neighborhoods of $U_{j}$s under $\mathcal{M}$
yield a support \scalebox{0.9}{$\lbrace\mathcal{I}_{j}\rbrace_{j=1}^{k}$}.
Conversely, any valid support yields a unique perfect matching in~$G$ (taking into account that all vertices in $U_{j}$ are isomorphic). 
Moreover, due to the choice of weights in~$G$,
the right-hand side of~\eqref{local-objective-upper-bound} for a given support \scalebox{0.9}{$\lbrace\mathcal{I}_{j}\rbrace_{j=1}^{k}$} is equal to the weight of the matching~$\mathcal{M}$ in~$G$ induced by the former,
\textit{i.e.},
$
	\scalebox{0.9}{$\sum_{j=1}^{{k}}\sum_{i \in \mathcal{I}_{j}} W_{ij}^{2}$}
	\scalebox{0.9}{$=$}
	\scalebox{0.9}{$\sum_{(u,v) \in \mathcal{M}} w(u,v)$}
$.
It follows that
\it 
determining the support of the solution in~\eqref{local-problem},
reduces to solving the maximum weight matching problem on the bipartite graph~$G$.
\normalfont
 
Algorithm~\ref{algo:local-candidate} readily follows.
Given $\W \in \mathbb{R}^{\dimension \times {k}}$, the algorithm generates a weighted bipartite graph~$G$ as described, and computes its maximum weight matching.
Based on the latter, it first recovers the desired support of \scalebox{0.9}{$\widehat{\X}$} ({Line 5}), and subsequently the exact values of its nonzero entries ({Line 6}).
The running time is dominated by the computation of the matching, which can be done in 
$O\mathopen{}\bigl(|E||U|+|U|^{2}\log{|U|}\bigr)$
using a variant of the Hungarian algorithm~\cite{ramshaw2012minimum}.
Hence,
\begin{lemma} 
	\label{lemma:algo-local-candidate-guarantees}
	For any ${\W \in \mathbb{R}^{\dimension \times {k}}}$, 
	Algorithm~\ref{algo:local-candidate} computes the solution to~\eqref{local-problem}, in time~${O\mathopen{}\bigl(\dimension \cdot ({s}\cdot {k})^{2}\bigr)}$.
\end{lemma}
A more formal analysis and proof of Lemma~\ref{lemma:algo-local-candidate-guarantees} is available in Section~\ref{proof:algo-local-candidate-guarantees}.
With Algorithm~\ref{algo:local-candidate} and Lemma~\ref{lemma:algo-local-candidate-guarantees} in place,
we complete the description of our sparse PCA algorithm (Algorithm~\ref{algo:lowrank-symmetric-multi-component})
and the proof sketch of Theorem~\ref{thm:spca-multi-rank-r-guarantees}.

\section{Sparse PCA on Low-Dimensional Sketches}
\label{sec:spca-on-sketch}

\setlength{\columnsep}{1.5em}
\begin{wrapfigure}[9]{R}[0pt]{0.52\textwidth}
\begin{minipage}{0.52\textwidth}
   \vspace{-2.0em}
\begin{algorithm}[H]
   \caption{Sparse PCA on Low Dim. Sketch}
   \label{algo:symmetric-multi-component}
   \begin{algorithmic}[1]
   		\INPUT: Real $n \times \dimension$ $\mathbf{S}$, $r \in \Z_{+}$, ${\epsilon \in (0,1)}$, ${k} \in \Z_{+}$.
		\OUTPUT $\widebar{\X}_{\mathsmaller{(r)}} \in \Xset_{k}$.
		\hfill \COMMENT{Thm.~\ref{thm:generic-sketch-solution}}
		\STATE $\widebar{\mathbf{S}} \gets \text{\textsc{Sketch}}(\mathbf{S}, r)$
   		\STATE $\widebar{\A} \gets 
		\widebar{\mathbf{S}}^{\T}\widebar{\mathbf{S}}$
		\STATE $\widebar{\X}_{\mathsmaller{(r)}} \gets \text{\textsc{Algorithm}~\ref{algo:lowrank-symmetric-multi-component}}$ $(\widebar{\A}, \epsilon, {k})$.
   \end{algorithmic}
\end{algorithm}
\end{minipage}
\end{wrapfigure}
Algorithm~\ref{algo:lowrank-symmetric-multi-component}
approximately solves the sparse PCA problem~\eqref{spca-multi-def}
on a $d \times d$ \mbox{rank-$r$} PSD matrix $\A$,
in time that grows as a low-order polynomial in the ambient dimension~$d$,
but depends exponentially on~$r$.
This dependence can be prohibitive in practice.
To mitigate its effect,
instead of the original input,
we can apply our sparse PCA algorithm on a low-rank approximation of $\A$. 
Intuitively, the quality of the extracted components should depend on how well that low-rank surrogate approximates the original input.

More formally, let $\mathbf{S}$ be the real ${n} \times \dimension$ data matrix representing~$n$ (potentially centered) datapoints in~$\dimension$ variables,
and~$\mathbf{A}$ the corresponding $\dimension \times \dimension$ covariance matrix. 
Further, let $\widebar{\mathbf{S}}$ be a low-dimensional sketch of the original data;
an ${n} \times \dimension$ matrix whose rows lie in an $r$-dimensional subspace, with $r$ being an accuracy parameter.
Such a sketch can be obtained in several ways,
including for example exact or approximate SVD, 
or online sketching methods \cite{halko2011finding}.
Finally, let $\widebar{\A} = \scalebox{0.85}{$\sfrac{1}{n}\cdot \widebar{\mathbf{S}}^{\T}\widebar{\mathbf{S}}$}$ be the covariance matrix of the sketched data.
Then, instead of $\A$,
we can approximately solve the sparse PCA problem by applying Algorithm~\ref{algo:lowrank-symmetric-multi-component} on the low-rank surrogate $\widebar{\A}$.
The above are formally outlined in Algorithm~\ref{algo:symmetric-multi-component}.
We note that the covariance matrix $\widebar{\A}$ does not need to be explicitly computed; Algorithm~\ref{algo:lowrank-symmetric-multi-component} can operate directly on the (sketched) input data matrix.

\begin{theorem} 
	\label{thm:generic-sketch-solution}
	For any $n \times \dimension$ input data matrix $\mathbf{S}$,
	 with corresponding empirical covariance matrix $\A=\sfrac{1}{n}\cdot \mathbf{S}^{\T}\mathbf{S}$,
	  any desired number of components~$k$, 
	and accuracy parameters~$\epsilon \in (0,1)$ and~$r$,
	Algorithm~\ref{algo:symmetric-multi-component} outputs ${\X_{\mathsmaller{(r)}} \in \Xset_{k}}$ such that
\begin{align}
	\trace\bigl( {\X}_{\mathsmaller{(r)}}^{\T} \A {\X}_{\mathsmaller{(r)}} \bigr) 
	\;\ge\;
	(1 - \epsilon) \cdot 
	\trace\bigl( \X_{\star}^{\T} \A \X_{\star} \bigr)
	-
	2 \cdot {k}\cdot \lambda_{1,s}(\A - \widebar{\A}),
   \nonumber
\end{align}
in time
$\Tsketch(r) + \Tsvd(r) + O\mathopen{}\bigl(\bigl(\tfrac{4}{\epsilon}\bigr)^{r \cdot {{k}}} \cdot {\dimension} \cdot ({s} \cdot {k})^{2}\bigr)$.
Here, $\Xstar \eqdef \argmax_{\X \in \Xset_{k}} \trace\left(\X^{\T} \A \X \right)$,
and $\lambda_{1,s}(\A)$ 
denotes the {\emph sparse eigenvalue}, \textit{i.e.}, the eigenvalue that corresponds to the principal $s$-sparse eigenvector of~$\A$.
\end{theorem}
The error $\lambda_{1,s}(\A - \widebar{\A})$ and in turn the tightness of the approximation guarantees hinges on the quality of the sketch $\widebar{\A}$.
Higher values of the parameter~$r$ (the rank of the sketch) can allow for a more accurate solution and tighter guarantees.
That is the case, for example, when the sketch is obtained through exact SVD.
In that sense, Theorem~\ref{thm:generic-sketch-solution} establishes a natural trade-off between the running time of Algorithm~\ref{algo:symmetric-multi-component} and the quality of the approximation guarantees. A formal proof of Theorem~\ref{thm:generic-sketch-solution}  is provided in Section~\ref{sec:proof-algo-general-matrix}.
 Observe that the error term itself is a sparse eigenvalue that is hard to approximate, however even loose bounds provide tight conditional approximation results, as we see next.

Using the main matrix approximation result of \cite{alon2013approximate}, the next theorem establishes that Algorithm~\ref{algo:symmetric-multi-component} can be turned into an additive PTAS.

\begin{theorem}{\label{thm:opt-epsk}}
Let $\A$ be a $d \times d$ positive semidefinite matrix with entries in $[-1,1]$,
$\mathbf{V}$ be a $d \times d$ matrix such that $\A = \matV\matV^\top$.
Further, let $\mathbf{R}$ be a random $d \times r$ matrix with entries drawn \textit{i.i.d.} according to $\mathcal{N}(0,1/r)$, and define 
$$
	\widebar{\A} \,\eqdef\, \matV{\bf R}{\bf R}^\top\matV^\top.
$$
For any constant $\epsilon \in (0,1]$,
let $r = O(\epsilon^{-2}\log d)$.
Then, 
for any desired sparsity $s$, and number of components $k = {O}(1)$,
Algorithm~\ref{algo:lowrank-symmetric-multi-component} with input argument $\widebar{\A}$
and accuracy parameter $\epsilon$, 
outputs ${\X_{\mathsmaller{(r)}} \in \Xset_{k}}$ such that
\begin{align}
	\trace\bigl( {\X}_{\mathsmaller{(r)}}^{\T} \A {\X}_{\mathsmaller{(r)}} \bigr) 
	\;\ge\;	
	\trace\bigl( \X_{\star}^{\T} \A \X_{\star} \bigr)
	-\epsilon \cdot s
   \nonumber
\end{align}
with probability at least $1-1/\text{poly}(d)$, in time $n^{O(\log(1/\epsilon)/\epsilon^2))}$.
\end{theorem}

\begin{remark}
Note that $\lambda_{1}(\A - \widebar{\A})$ serves as another elementary upper bound on $\lambda_{1,s}(\A - \widebar{\A})$. If $\widebar{\A}$ is a the rank-$d$ SVD approximation of $\A$, then---similar to~\cite{asteris2014nonnegative}---we can obtain a multiplicative PTAS for sparse PCA, under the assumption of a decaying spectrum (e.g., under a power-law decay), and for $s = \Omega(n)$.
\end{remark}
\section{Related Work}
\label{sec:relatedwork}

We are not aware of any algorithm with provable guarantees for sparse PCA with {\it disjoint supports}.
Multiple components can be extracted by repeatedly solving~\eqref{spca-single-def} using one of the aforementioned methods.
To ensure disjoint supports, variables ``selected" by a component are removed from the dataset. 
This greedy approach, however, can result in highly suboptimal objective value (See example in Sec.~\ref{sec:deflation-suboptimality}).

A significant volume of work has focused on the single-component sparse PCA problem~\eqref{spca-single-def}; we scratch the surface and refer the reader to citations therein. 
Representative examples range from early heuristics in \cite{jolliffe1995rotation}, 
to the LASSO based techniques in \cite{jolliffe2003modified},  
the elastic net $\ell_1$-regression in \cite{zou2006sparse},
$\ell_1$ and $\ell_0$ regularized optimization methods such as GPower in \cite{journee2010generalized},
a greedy branch-and-bound technique in \cite{moghaddam2006spectral}, 
or semidefinite programming approaches \cite{d2008optimal, zhang2012sparse, d2007direct}.
The authors of \cite{sigg:2008} present an approach that uses ideas from an expectation-maximization (EM) formulation of the problem.
More recently,~\cite{yuan2013truncated} presents a simple and very efficient truncated version of the power iteration (TPower).
Finally,~\cite{spanspca:IT2014} introduces an exact solver for the low-rank case of the problem; this solver was then used on low-rank sketches in the work of \cite{papailiopoulos2013sparse} (SpanSPCA), that provides conditional approximation guarantees under spectral assumptions on the input data. Several ideas in this work are inspired by the aforementioned low-rank solvers. 
In our experiments, we compare against EM, TPower, and SpanSPCA, which all are experimentally achieving state-of-the-art performance.

Parallel to the algorithmic and optimization perspective, there is large line of statistical analysis for sparse PCA that focuses on guarantees pertaining to planted models and the recovery of a ``true" sparse component
\cite{amini2008high,
ma2013sparse,
d2012approximation,
cai2012sparse,
deshpande2013sparse,
berthet2013optimal,
berthet2013complexity,
WBS14-more_SPCA_planted-hardness,
KNV15-SDP-Sparse_PCA}.

There has been some work on the explicit estimation of principal subspaces or multiple components under sparsity constraints.
Non-deflation-based algorithms include extensions of the diagonal thresholding algorithm \cite{johnstone2009consistency} and iterative thresholding approaches \cite{ma2013sparse},
while \cite{vu2013fantope} and 
~\cite{wang2014nonconvex} propose methods that rely on the ``row sparsity for subspaces" assumption of \cite{vu2012minimax}.
These methods yield components supported on a common set of variables, and hence solve a problem different from~\eqref{spca-multi-def}.
Magdon-Ismail and Boutsidis~\cite{magdon:linearencoders} discuss the multiple component Sparse PCA problem, propose an alternative objective function and for that problem obtain interesting theoretical guarantees. 
Finally, \cite{richard2014tight} develops a framework for sparse matrix factorizaiton problems, based on a novel atomic norm. 
That framework captures sparse PCA -- although not explicitly the constraint of disjoint supports -- but the resulting optimization problem, albeit convex, is NP-hard.

\renewcommand{\bottomrulec}{%
  \arrayrulecolor{black}\specialrule{\heavyrulewidth}{0pt}{0pt}
  \arrayrulecolor{black}
}
\renewcommand{\toprulec}{%
  \arrayrulecolor{black}\specialrule{\heavyrulewidth}{\aboverulesep}{0pt}
  \arrayrulecolor{black}
}
\renewcommand{\midrulec}{%
  \arrayrulecolor{white}\specialrule{\aboverulesep}{0pt}{0pt}
  \arrayrulecolor{black}\specialrule{\lightrulewidth}{0pt}{0pt}
}

\newcounter{magicrownumbers}
\newcommand{\rownumber}{%
\ifnum\themagicrownumbers>0%
	{\color{gray}\scriptsize\arabic{magicrownumbers}:}%
\else{\relax}%
\fi%
\stepcounter{magicrownumbers}%
}
\definecolor{mustard}{RGB}{218, 160, 41}

\section{Experiments}
\label{sec:experiments}

We evaluate our algorithm on a series of real datasets, and compare it to deflation-based approaches for sparse PCA using
TPower~\cite{yuan2013truncated},
EM~\cite{sigg:2008}, and 
SpanSPCA~\cite{papailiopoulos2013sparse}.
The latter are representative of the state of the art for the single-component sparse PCA problem~\eqref{spca-single-def}.
Multiple components are computed one by one.
To ensure disjoint supports, 
the deflation step effectively amounts to removing from the dataset all variables used by previously extracted components.
For algorithms that are randomly initialized, we depict best results over multiple random restarts.
Additional experimental results are listed in Section~\ref{sec:apndx-experiments} of the appendix.

Our experiments are conducted in a Matlab environment. 
Due to its nature, our algorithm is easily parallelizable;
its prototypical implementation utilizes the Parallel Pool Matlab feature to exploit multicore (or distributed cluster) capabilities. 
Recall that our algorithm operates on a low-rank approximation of the input data.
Unless otherwise specified, it is configured for a rank-$4$ approximation obtained via truncated SVD. 
Finally, we put a time barrier in the execution of our algorithm, at the cost of the theoretical approximation guarantees;
the algorithm returns best results at the time of termination.
This ``{early termination}" can only hurt the performance of our algorithm.

\textbf{Leukemia Dataset.}\;
We evaluate our algorithm on the Leukemia dataset~\cite{Lichman:2013}.
The dataset comprises $72$ samples, each consisting of expression values for $12582$ probe sets.
We extract ${k=5}$ sparse components, each active on $s=50$ features. 
In Fig.~\ref{leukemia:cvar-line:s50}, we plot the cumulative explained variance versus the number of components. 
Deflation-based approaches are greedy: 
the leading components capture high values of variance, but subsequent ones contribute less.
On the contrary, our algorithm jointly optimizes the ${k=5}$ components and achieves higher \emph{total} cumulative variance;
one cannot identify a top component.
We repeat the experiment for multiple values of~$k$.
Fig.~\ref{leukemia:cvar-barplot:s50} depicts the total cumulative variance capture by each method, for each value of~$k$. 
\begin{figure}[tb!]
	\centering
   \subfigure[tight][]{
	  \includegraphics[height=0.37\textwidth, trim=0cm 0pt 0cm 0cm, clip=true]{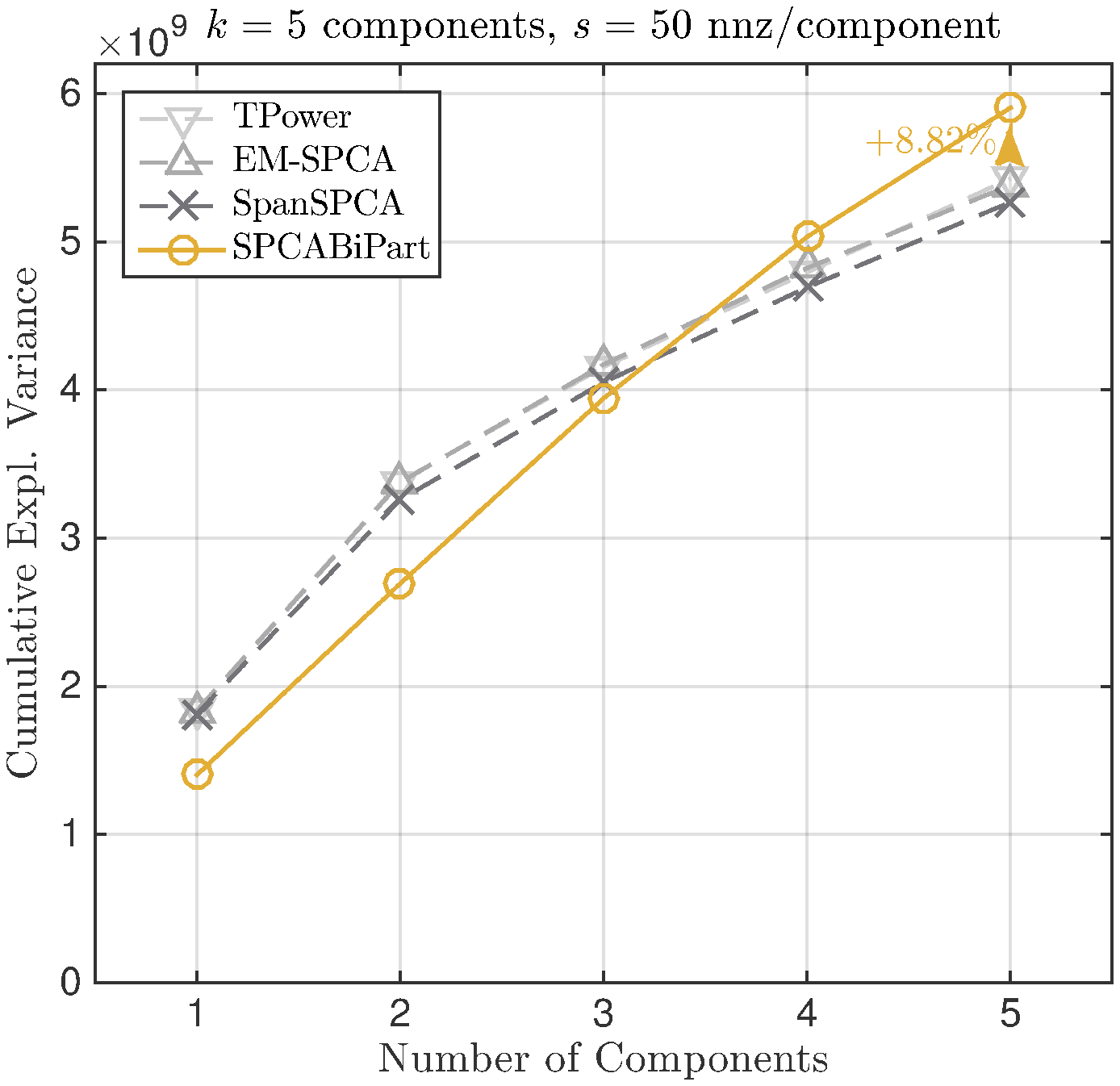}
	  \label{leukemia:cvar-line:s50}
   }
   \subfigure[tight][]{
	  \includegraphics[height=0.37\textwidth, trim=0cm 0pt 0cm 0cm, clip=true]{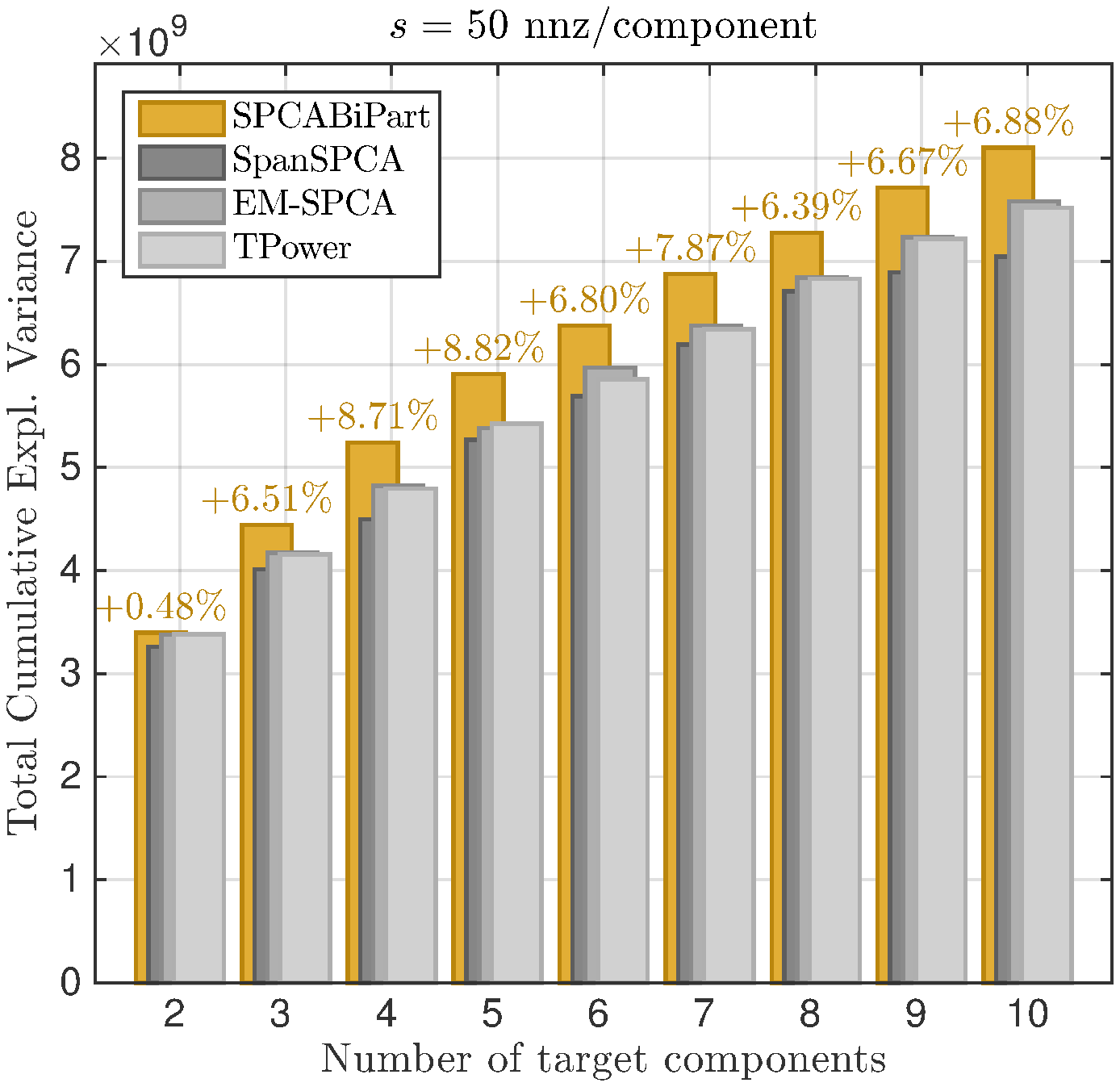}
	  \label{leukemia:cvar-barplot:s50}
   }
   \vspace{-1.0em}
	\caption{
	   	Cumulative variance captured by the $k$ $s$-sparse extracted components -- Leukemia dataset~\cite{Lichman:2013}.
	Sparsity is arbitrarily set to ${s=50}$ nonzero entries per component. 
		Fig.~\ref{leukemia:cvar-line:s50} depicts the cum. variance versus the number of components, for ${k=5}$. 
Deflation-based approaches are greedy;
first components capture high variance, but subsequent ones contribute less. 
	Our algorithm jointly optimizes the ${k=5}$ components and achieves higher \emph{total} cum. variance. 
	Fig.~\ref{leukemia:cvar-barplot:s50} depicts the total cum. variance achieved for various values of~$k$. 
	}
	\label{leukemia-figures}
\end{figure}

\textbf{Additional Datasets.}\,
We repeat the experiment on multiple datasets,
arbitrarily selected from~\cite{Lichman:2013}.
Table~\ref{spca-various-datasets:k5s40} lists the total cumulative variance captured by ${k=5}$ components, each with ${s=40}$ nonzero entries, 
extracted using the four methods. 
Our algorithm achieves the highest values in most cases. 
\renewcommand{\arraystretch}{1.1}
\begin{table}[b!]
	\fontsize{8.3}{10}\selectfont
	\centering
	\rowcolors{2}{mustard!20}{white}
	\begin{tabular}{%
		>{\hspace{-\tabcolsep plus .5em}}l<{}%
		>{\hspace{-\tabcolsep}\color{black!65}$(}r<{\times$\hspace{-\tabcolsep}}%
		>{\hspace{-\tabcolsep}\color{black!65}$}l<{)$}%
		>{$}c<{$}%
		>{$}c<{$}%
		>{$}c<{$}%
		>{$}c<{$\hspace{-\tabcolsep}}%
	}
	\toprulec
	\multicolumn{3}{c}{} & \text{TPower} & \text{EM sPCA} & \text{SpanSPCA} & \multicolumn{1}{c}{\text{SPCABiPart}} \\ 
	\midrulec
	\midrulec
   \textsc{Amzn Com Rev} &1500 &  10000&       7.31e+03 &      7.32e+03 &      7.31e+03 &     \mathbf{ 7.79e+03 }\\
  \textsc{Arcence Train} &100 &  10000&       1.08e+07 &      1.02e+07 &      1.08e+07 &     \mathbf{ 1.10e+07 }\\
  \textsc{CBCL Face Train} &   2429 &    361&       5.06e+00 &      5.18e+00 &      5.23e+00 &     \mathbf{ 5.29e+00 }\\
       \textsc{Isolet-5} &   1559 &    617&       3.31e+01 &      3.43e+01 &      3.34e+01 &     \mathbf{ 3.51e+01 }\\
	\textsc{Leukemia} &72 &  12582&       5.00e+09 &      5.03e+09 &      4.84e+09 &     \mathbf{ 5.37e+09 }\\
	     \textsc{Pems Train} &   267 & 138672&      \mathbf{ 3.94e+00 } &   3.58e+00 &      3.89e+00 &      3.75e+00 \\
      \textsc{Mfeat Pix} &2000 &    240&       5.00e+02 &      5.27e+02 &      5.08e+02 &     \mathbf{ 5.47e+02 } \\
 	\bottomrulec
   \end{tabular}
      \caption{
      Total cumulative variance captured by ${k=5}$ $40$-sparse extracted components on various datasets~\cite{Lichman:2013}.
   For each dataset, we list the size (\#samples$\times$\#variables) and the value of  variance captured by each method. 
   Our algorithm operates on a rank-${4}$ sketch in all cases.
   }
   \label{spca-various-datasets:k5s40}
\end{table}

\renewcommand{\arraystretch}{1.1}
\begin{table}[t!]
	\fontsize{8.3}{10}\selectfont
	\centering
	\rowcolors{2}{mustard!20}{white}
	\begin{tabular}{%
		>{\hspace{-\tabcolsep plus .5em}}l<{}%
		>{\hspace{-\tabcolsep}\color{black!65}$(}r<{\times$\hspace{-\tabcolsep}}%
		>{\hspace{-\tabcolsep}\color{black!65}$}l<{)$}%
		>{$}c<{$}%
		>{$}c<{$}%
		>{$}c<{$}%
		>{$}c<{$\hspace{-\tabcolsep}}%
	}
	\toprulec
	\multicolumn{3}{c}{} & \text{TPower} & \text{EM sPCA} & \text{SpanSPCA} & \multicolumn{1}{c}{\text{SPCABiPart}} \\ 
	\midrulec
	\midrulec
\textsc{BoW:NIPS} &   1500 &  12419&       2.51e+03 &      2.57e+03 &      2.53e+03 &     \mathbf{ 3.34e+03 }\;(+29.98\%)\\
\textsc{BoW:KOS} &   3430 &   6906&       4.14e+01 &      4.24e+01 &      4.21e+01 &     \mathbf{ 6.14e+01 }\;(+44.57\%)\\
  \textsc{BoW:Enron} &  39861 &  28102&       2.11e+02 &      2.00e+02 &      2.09e+02 &     \mathbf{ 2.38e+02 }\;(+12.90\%)\\
  \textsc{BoW:NyTimes} & 300000 & 102660&       4.81e+01 &      - &      4.81e+01 &     \mathbf{ 5.31e+01 }\;(+10.38\%)\\
   	\bottomrulec
   \end{tabular}
      \caption{
      Total variance captured by ${k=8}$ extracted components, each with ${s=15}$ nonzero entries -- Bag of Words dataset~\cite{Lichman:2013}.
   For each corpus, we list the size (\#documents$\times$\#vocabulary-size) and the explained variance.
   Our algorithm operates on a rank-${5}$ sketch in all cases.
   }
   \label{spca-text-datasets}
\end{table}

\textbf{Bag of Words (BoW) Dataset.}~\cite{Lichman:2013}\,
This is a collection of text corpora stored under the ``bag-of-words" model.
For each text corpus, a vocabulary of ${d}$ words is extracted upon tokenization, and the removal of stopwords and words appearing fewer than ten times in total.
Each document is then represented as a vector in that $d$-dimensional space, with the $i$th entry corresponding to the number of appearances of the $i$th vocabulary entry in the document.

We solve the sparse PCA problem~\eqref{spca-multi-def} on the word-by-word cooccurrence matrix,
and extract ${k=8}$ sparse components, each with cardinality ${s=10}$.
We note that the latter is not explicitly constructed; our algorithm can operate directly on the input word-by-document matrix. 
Table~\ref{spca-text-datasets} lists the variance captured by each method; our algorithm consistently outperforms the other approaches. 

Finally, note that here
each sparse component effectively \emph{selects} a small set of words.
In turn, the $k$ extracted components can be interpreted as a set of well-separated \emph{topics}.
In Table~\ref{nytimes-topics-our},
we list the topics extracted from the NY Times corpus (part of the Bag of Words dataset).
The corpus consists of ${3 \cdot 10^5}$ news articles and a vocabulary of ${d=102660}$ words.

\renewcommand{\arraystretch}{1.1}
\begin{table}[bh!]
	\setcounter{magicrownumbers}{0}
	\fontsize{8.3}{10}\selectfont
	\centering
	\setlength{\tabcolsep}{0.7\tabcolsep}
	\rowcolors{2}{white}{mustard!20!white}
	\begin{tabular}{%
		>{\hspace{-\tabcolsep plus .5em}{\makebox[1.3em][r]{\rownumber\space}}\hspace{.3em}}l<{}%
		>{}l<{}%
		>{}l<{}%
		>{}l<{}%
		>{}l<{}%
		>{}l<{}%
		>{}l<{}%
		>{}l<{\hspace{-\tabcolsep}}%
	}
	\toprulec
	Topic $1$ &
	Topic $2$ &
	Topic $3$ &
	Topic $4$ &
	Topic $5$ &
	Topic $6$ &
	Topic $7$ &
	Topic $8$ \\
	\midrulec
	\midrulec
percent&        zzz\_united\_states&    zzz\_bush&      company&        team&   cup&    school& zzz\_al\_gore\\
million&        zzz\_u\_s&      official&       companies&      game&   minutes&        student&        zzz\_george\_bush\\
money&  zzz\_american&  government&     market& season& add&    children&       campaign\\
high&   attack& president&      stock&  player& tablespoon&     women&  election\\
program&        military&       group&  business&       play&   oil&    show&   plan\\
number& palestinian&    leader& billion&        point&  teaspoon&       book&   tax\\
need&   war&    country&        analyst&        run&    water&  family& public\\
part&   administration& political&      firm&   right&  pepper& look&   zzz\_washington\\
problem&        zzz\_white\_house&      american&       sales&  home&   large&  hour&   member\\
com&    games&  law&    cost&   won&    food&   small&  nation\\
	\bottomrulec
   \end{tabular}
   \caption{
  \textsc{ BoW:NyTimes} dataset~\cite{Lichman:2013}.
   The table lists the words corresponding to the ${s=10}$ nonzero entries of each of the 
   $k=8$ extracted components (topics). Words corresponding to higher magnitude entries appear higher in the topic.
   }
   \label{nytimes-topics-our}
\end{table}


\section{Conclusions}
\label{sec:conclusions}
We considered the sparse PCA problem for multiple components with disjoint supports. 
Existing methods for the single component problem can be used along with an appropriate deflation step to compute multiple components one by one,
leading to potentially suboptimal results. 
We presented a novel algorithm for jointly optimizing multiple sparse and disjoint components with provable approximation guarantees.
Our algorithm is combinatorial and exploits interesting connections between the sparse PCA and the bipartite maximum weight matching problems. 
It runs in time that grows as a low-order polynomial in the ambient dimension of the input data, but depends exponentially on its rank.
To alleviate this dependency, we can apply the algorithm on a low-dimensional sketch of the input,
at the cost of an additional error in our theoretical approximation guarantees.
Empirical evaluation of our algorithm demonstrated that in many cases it outperforms deflation-based approaches.

\section*{Acknowledgments}
DP is generously supported by NSF awards CCF-1217058 and CCF-1116404 and MURI AFOSR grant 556016. 
This research has been supported by NSF Grants CCF 1344179, 1344364, 1407278, 1422549 and ARO YIP W911NF-14-1-0258.

\bibliographystyle{ieeetr}
\bibliography{spcamulti}

\section*{Supplemental Material}
\section{On the sub-optimality of deflation -- An example}
\label{sec:deflation-suboptimality}
We provide a simple example demonstrating the sub-optimality of deflation based approaches for computing multiple sparse components with disjoint supports. 
Consider the real $4 \times 4$ matrix
\begin{align}
	\mathbf{A}
	=
	\begin{bmatrix}
		1 & 0 & 0 & \epsilon \\
		0 & \delta & 0 & 0 \\
		0 & 0 & \delta & 0 \\
		\epsilon & 0 & 0 & 1
	\end{bmatrix},
	\nonumber
\end{align}
with $\epsilon,\delta > 0$ such that ${\epsilon+\delta < 1}$.
Note that $\mathbf{A}$ is PSD; 
$\mathbf{A} = \mathbf{B}^{\transpose}\mathbf{B}$ for 
\begin{align}
	\mathbf{B}
	=
	\begin{bmatrix}
		1 & 0 & 0 & \epsilon \\
		0 & \sqrt{\delta} & 0 & 0 \\
		0 & 0 & \sqrt{\delta} & 0 \\
		0 & 0 & 0 & \sqrt{1-\epsilon^2}
	\end{bmatrix}.
	\nonumber
\end{align}
We seek two $2$-sparse components with disjoint supports, 
\textit{i.e.}, the solution to 
\begin{align}
	\max_{\mathbf{X} \in \mathcal{X}}
		\sum_{j=1}^{2}
	\mathbf{x}_{j}^{\transpose}
	\mathbf{A}
	\mathbf{x}_{j},
	\label{2-components-opt-problem}
\end{align}
where
\begin{align}
	\Xset \eqdef 
	\left\lbrace 
		\X \in \mathbb{R}^{4 \times 2}: 
		\|\mathbf{x}_{i}\|_{2} \le 1, 
		\|\mathbf{x}_{i}\|_{0} \le 2 \;\forall\, i \in \lbrace{1,2}\rbrace,
		\supp(\mathbf{x}_{1}) \cap \supp(\mathbf{x}_{2}) = \emptyset
	\right\rbrace.
	\nonumber
\end{align}

\textbf{Iterative computation with deflation.}
Following an iterative, greedy procedure with a deflation step, we compute one component at the time. 
The first component is
\begin{align}
	\mathbf{x}_{1}
	=
	\argmax_{
		\substack{
			\|\mathbf{x}\|_{0}=2,
			\|\mathbf{x}\|_{2}=1
		}
	}	
	\mathbf{x}^{\transpose} \mathbf{A} \mathbf{x}.
	\label{greedy-x1}
\end{align}
Recall that for any unit norm vector $\mathbf{x}$ with support $I = \supp(\mathbf{x})$, 
\begin{align}
	\mathbf{x}^{\transpose} \mathbf{A} \mathbf{x}
	\le
	\lambda_{\max}\left( \mathbf{A}_{I,I}\right),
	\label{single-component-obj}
\end{align}
where $\mathbf{A}_{I,I}$ denotes the principal submatrix of $\mathbf{A}$ formed by the rows and columns indexed by $I$.
Equality can be achieved in~\eqref{single-component-obj} for $\mathbf{x}$ equal to the leading eigenvector of $\mathbf{A}_{I,I}$.
Hence, it suffices to determine the optimal support for $\mathbf{x}_{1}$.
Due to the small size of the example, 
it is easy to determine that the set~$I_{1}=\lbrace 1, 4 \rbrace$ maximizes the objective in~\eqref{single-component-obj} over all sets of two indices,
achieving value
\begin{align}
	\mathbf{x}_{1}^{\transpose}
	\mathbf{A}
	\mathbf{x}_{1}
	=
	\lambda_{\max}\left(\begin{bmatrix} 1 & \epsilon \\ \epsilon & 1\end{bmatrix}\right) = 1+\epsilon.
	\label{greedy-x1-contribution}
\end{align}
Since subsequent components must have disjoint supports, 
it follows that the support of the second $2$-sparse component $\mathbf{x}_{2}$ is~$I_{2}=\lbrace 2,3\rbrace$,
and $\mathbf{x}_{2}$ achieves value
\begin{align}
	\mathbf{x}_{2}^{\transpose}
	\mathbf{A}
	\mathbf{x}_{2}
	=
	\lambda_{\max}\left(\begin{bmatrix} \delta & 0 \\ 0 & \delta \end{bmatrix}\right) 
	= \delta.
	\label{greedy-x2-contribution}
\end{align}
In total, the objective value in~\eqref{2-components-opt-problem} achieved by the greedy computation with a deflation step is 
\begin{align}
	\sum_{j=1}^{2}
		\mathbf{x}_{j}^{\transpose}\mathbf{A}\mathbf{x}_{j}
	=
	1 + \epsilon + \delta.
	\label{first-alt-objective}
\end{align}
\textbf{The sub-optimality of deflation.}
Consider an alternative pair of $2$-sparse components~${\mathbf{x}_{1}^{\prime}}$ and~${\mathbf{x}_{2}^{\prime}}$
with support sets  ${I_{1}^{\prime}=\lbrace 1,2\rbrace}$ and~${I_{2}^{\prime}=\lbrace 3,4 \rbrace}$, respectively. 
Based on the above, such a pair achieves objective value in~\eqref{2-components-opt-problem} equal to 
\begin{align}
	\lambda_{\max}\left(\begin{bmatrix}1&0\\ 0 & \delta \end{bmatrix}
	\right)
	+
	\lambda_{\max}\left(\begin{bmatrix}\delta&0\\ 0 & 1 \end{bmatrix}\right)
	=
	1+1 = 2,
	\nonumber
\end{align}
which clearly outperforms the objective value in~\eqref{first-alt-objective} (under the assumption $\epsilon+\delta<1$), demonstrating the sub-optimality of the $\mathbf{x}_{1}$, $\mathbf{x}_{2}$ pair computed by the deflation-based approach.
In fact, for small $\epsilon, \delta$ the objective value in the second case is larger than the former by almost a factor of two.

\section{Construction of Bipartite Graph}
\label{sec:gen-bipart-graph}
The following algorithm formally outlines the steps for generating the bipartite graph ${G= \bigl(\lbrace{U_{j}}\rbrace_{j=1}^{{k}},V,E \bigr)}$ given a \emph{weight} $\dimension \times {k}$ matrix $\W$.

\setlength{\columnsep}{1.5em}
\begin{minipage}{1\textwidth}
   \vspace{-1em}
\begin{algorithm}[H]
	\caption{Generate Bipartite Graph}
	\label{algo:graph-gen}
	 	\begin{algorithmic}[1]
			\INPUT Real $\dimension \times {k}$ matrix $\W$
			\OUTPUT Bipartite ${G= \bigl(\lbrace{U_{j}}\rbrace_{j=1}^{{k}},V,E \bigr)}$
			\hfill \COMMENT{Fig.~\ref{fig:spca-bipartite}}
			\FOR{$j=1,\hdots, {k}$}
				\STATE $U_{j} \gets \bigl\lbrace u_{1}^{(j)}, \hdots, u_{{s}}^{(j)}\bigr\rbrace$
			\ENDFOR
			\STATE $U \gets \cup_{j=1}^{{k}} U_{j}$
			\hfill \COMMENT{$|U| = {k} \cdot {s}$}
			\STATE $V \gets \bigl\lbrace 1, \hdots, \dimension \bigr\rbrace$
			\STATE $E \gets U \times V$
			\FOR{$i=1, \hdots, \dimension$}
				\FOR{$j=1, \hdots, {k}$} 
					\FOR{ \textbf{each} $u \in U_{j}$} 
						\STATE	$w\bigl( u,  v_{i}\bigr) \gets W_{ij}^{2}$
					\ENDFOR
				\ENDFOR
			\ENDFOR
		\end{algorithmic}
\end{algorithm}
\end{minipage}

\section{Proofs}

\subsection{Guarantees of Algorithm~\ref{algo:local-candidate}}
\label{proof:algo-local-candidate-guarantees}
\begin{replemma}{lemma:algo-local-candidate-guarantees}
	\label{lemma:algo-local-guarantees}
	For any real $\dimension \times {k}$ matrix  $\mathbf{W}$, and 
	Algorithm~\ref{algo:local-candidate} outputs
	\begin{align}
		\widetilde{\X} 
		= \argmax_{\X \in \Xset_{k}}
		\sum_{j=1}^{{k}}
		\bigl\langle \mathbf{X}^{j}, \mathbf{W}^{j}\bigr\rangle^{2}
		\label{local-objective}
	\end{align}
	in time $O\mathopen{}\left(\dimension \cdot ({s} \cdot {k})^{2}\right)$.
\end{replemma}
\begin{proof}
	Consider a matrix $\X \in \Xset_{k}$
	and let $I_{j}$, $j=1\hdots, {k}$
	denote the support sets of its columns.
	By the constraints in $\Xset_{k}$, those sets are disjoint,
	\textit{i.e.},  $I_{j_{1}} \cap I_{j_{2}} = \emptyset$ $\forall j_{1}, j_{2} \in \range{{k}}, j_{1} \neq j_{2}$,
	and
	\begin{align}
		\sum_{j=1}^{{k}}
		\bigl\langle
			\mathbf{X}^{j},\, \mathbf{W}^{j}
		\bigr\rangle^{2}
		=
		\sum_{j=1}^{{k}}
		\Bigl(
			\sum_{i \in I_{j}} X_{ij} \cdot W_{ij}
		\Bigr)^{2}
		\le
		\sum_{j=1}^{{k}}
			\Bigl( \sum_{i \in I_{j}}  W_{ij}^{2} \Bigr).
		\label{spca-proof-upper-bound}
	\end{align}
	The last inequality is due to Cauchy-Schwarz and the fact that $\|\mathbf{X}^{j}\|_{2} \le 1$, $\forall\, j \in \range{{k}}$.
	In fact, if the supports sets $I_{j}$, ${j=1,\hdots,{k}}$ were known,
	the upper bound in~\eqref{spca-proof-upper-bound} would be achieved by setting ${\X}_{I_{j}}^{j} = \mathbf{W}_{I_{j}}^{j} / \|\mathbf{W}_{I_{j}}^{j} \|_{2}$, 
	\textit{i.e.}, setting the nonzero subvector of the $j$th column of ${\X}$ colinear to the corresponding subvector of the $j$th column of $\mathbf{W}$.
	Hence, the key step towards computing the optimal solution $\widetilde{\X}$ 
	is to determine the support sets $I_{j}$, $j=1,\hdots, {k}$ of its columns.
	
	Consider the set of binary matrices
	\begin{align}
	  \mathcal{Z} 
	  \eqdef
	  \left\lbrace 
		 \mathbf{Z} \in \lbrace{0,1}\rbrace^{\dimension \times {k}}: 
		 \| \mathbf{Z}^{j} \|_{0} \le s \;\forall\, j \in [{k}],
		 \supp(\mathbf{Z}^{i}) 
		 \cap 
		 \supp(\mathbf{Z}^{j}) = \emptyset
		 \;\forall\, i,j \in [{k}], i \neq j
	  \right\rbrace.
	  \nonumber
	\end{align}
	The set represents all possible supports for the members of $\Xset_{k}$.
	Taking into account the previous discussion, 
	the maximization in~\eqref{local-objective} can be written with respect to ${\mathbf{Z} \in \mathcal{Z}}$:
	\begin{align}
		\max_{\X \in \Xset_{k}}
		\sum_{j=1}^{{k}} \bigl\langle \mathbf{X}^{j},\,\mathbf{W}^{j}\bigr\rangle^{2}
		=
		\max_{\mathbf{Z} \in \mathcal{Z}}
		\sum_{j=1}^{{k}} \sum_{i=1}^{\dimension} Z_{ij} W_{ij}^{2}.
		\label{local-equivalent-support}
	\end{align}
	Let 
	$\widetilde{\mathbf{Z}} \in \mathcal{Z}$ denote the optimal solution, 
	which corresponds to the (support) indicator of $\widetilde{\X}$.
	Next, we show that computing $\widetilde{\mathbf{Z}}$ boils down to solving a maximum weight matching problem on the bipartite graph generated by Algorithm~\ref{algo:graph-gen}.	
	Recall that given $\mathbf{W} \in \mathbb{R}^{\dimension \times {k}}$,
    Algorithm~\ref{algo:graph-gen} generates a complete weighted bipartite graph $G=(U,V,E)$ where
\begin{itemize}[noitemsep, leftmargin=*]
	\item $V$ is a set of $\dimension$ vertices $v_{1}, \hdots, v_{\dimension}$, corresponding to the $\dimension$ variables, 
	\textit{i.e.}, the $\dimension$ rows of \scalebox{0.9}{$\widehat{\X}$}.
	\item $U$ is a set of ${k}\cdot {s}$ vertices, conceptually partitioned into ${k}$ disjoint subsets $U_{1}, \hdots, U_{{k}}$, each of cardinality~${s}$.
	The $j$th subset,~$U_{j}$, is associated with the support~$\mathcal{I}_{j}$;
	the~${s}$ vertices $u^{\mathsmaller{(j)}}_{\alpha}$, ${\alpha=1,\hdots,s}$ in~$U_{j}$ serve as placeholders for the variables/indices in~$\mathcal{I}_{j}$.
	\item Finally, the edge set is ${E = U \times V}$. 
	The edge weights are determined by the $\dimension \times {k}$ matrix $\W$ in~\eqref{local-problem}.
	In particular, the weight of edge $(u^{\mathsmaller{(j)}}_{\alpha}, v_{i})$ is equal to~$W_{ij}^{2}$.
	Note that all vertices in $U_{j}$ are effectively identical; they all share a common neighborhood and edge weights.
\end{itemize}
It is straightforward to verify that any $\mathbf{Z} \in \mathcal{Z}$ corresponds to a {perfect matching} in $G$ and vice versa;
$Z_{ij}=1$ if and only if vertex $v_{i} \in V$ is matched with a vertex in $U_{j}$ (all vertices in $U_{j}$ are equivalent with respect to their neighborhood).
Further, the objective value in~\eqref{local-equivalent-support} for a given $\mathbf{Z} \in \mathcal{Z}$ is equal to the weight of the corresponding matching in $G$. 
More formally, 
\begin{itemize}[noitemsep, leftmargin=*]
 \item Given a perfect matching $\mathcal{M}$, the support $I_{j}$ of the $j$th column of $\mathbf{Z}$ is determined by the neighborhood of $U_{j}$ in the matching:
   \begin{align}
   I_{j} \gets \bigl\lbrace i \in [\dimension]: (u, v_{i}) \in \mathcal{M}, u \in U_{j} \bigr\rbrace, \quad j=1,\hdots, {k}.
	  \label{support-construction-from-M}
   \end{align}
   Note that the sets $I_{j}$, $j=1,\hdots, {k}$ are indeed disjoint, and each has cardinality equal to ${s}$.
 The weight of the matching $\mathcal{M}$ is
   \begin{align}
	  \sum_{(u,v) \in \mathcal{M}} w(u,v)
	  =
	  \sum_{j=1}^{{k}}
	  \sum_{\substack{
			(u,v_{i}) \in \mathcal{M}:\\
			u \in U_{j}
			}} w(u, v_{i})
	  =
	  \sum_{j=1}^{{k}}
	  \sum_{i \in I_{j}} W_{ij}^{2}
	  =
	  \sum_{j=1}^{{k}}\sum_{i=1}^{\dimension} Z_{ij} \cdot W_{ij}^{2},
	  \label{objective-Z-M}
   \end{align}
   which is equal to the objective function in~\eqref{local-equivalent-support}.
 \item Conversely, given an indicator matrix $\mathbf{Z} \in \mathcal{Z}$,
 let $I_{j} \eqdef \supp(\mathbf{Z}^{j})$, and let $I_{j}(\alpha)$ denote the $\alpha$th element in the set, $\alpha=1,\hdots, {s}$ (with an arbitrary ordering).
 Then, 
 \begin{align}
   \mathcal{M}
   =
   \left\lbrace
	  (u^{\mathsmaller{(j)}}_{\alpha}, v_{I_{j}(\alpha)}),
	  \alpha=1,\hdots, {s}, \,
	  j=1,\hdots, {k}
   \right\rbrace
   \subset E
   \nonumber
 \end{align}
 is a perfect matching in $G$.
The objective value achieved by $\mathbf{Z}$ is equal to the weight of $\mathcal{M}$:
 \begin{align}
   \sum_{j=1}^{{k}}\sum_{i=1}^{\dimension} Z_{ij} \cdot W_{ij}^{2}
   =
   \sum_{j=1}^{{k}}
   \sum_{i \in I_{j}} W_{ij}^{2}
   =
   \sum_{j=1}^{{k}}
   \sum_{\alpha=1}^{s} W_{I_{j}(\alpha),j}^{2}
   =
   \sum_{(u,v) \in \mathcal{M}} w(u,v).
   \label{objective-M-Z}
 \end{align}
\end{itemize} 
It follows from~\eqref{objective-Z-M} and \eqref{objective-M-Z} that to determine $\widetilde{\mathbf{Z}}$,
it suffices to compute a maximum weight perfect matching in $G$.
The desired support is then obtained as described in~\eqref{support-construction-from-M} (lines 4-7 of Algorithm~\ref{algo:local-candidate}).
This complete the proof of correctness of Algorithm~\ref{algo:local-candidate} which proceeds in the steps described above to determine the support of $\widetilde{\X}$.

The weighted bipartite graph $G$ is generated in $O(\dimension \cdot ({s} \cdot {k}))$.
The running time of Algorithm~\ref{algo:local-candidate} is dominated by computing the maximum weight matching of $G$.
For the case of unbalanced bipartite graph with $|U|={{s} \cdot {k}}  < \dimension = \lvert{V}\rvert$ the Hungarian algorithm can be modified~\cite{ramshaw2012minimum} to compute the maximum weight bipartite matching in time 
$
O\mathopen{}\left(|E||U|+|U|^{2}\log{|U|}\right) 
=
{O\mathopen{}\bigl(\dimension \cdot ( {s} \cdot {k})^{2} \bigr)}
$.
This completes the proof. 
\end{proof}

\subsection{Guarantees of Algorithm~\ref{algo:lowrank-symmetric-multi-component} -- Proof of Theorem~\ref{thm:spca-multi-rank-r-guarantees}}
\label{sec:proof-algo-rank-r-spca}

We first prove a more general version of Theorem~\ref{thm:spca-multi-rank-r-guarantees} for arbitrary constraint sets. 
Combining that with the guarantees of Algorithm~\ref{algo:local-candidate}, we prove the Theorem~\ref{thm:spca-multi-rank-r-guarantees}.
\begin{lemma}
\label{lowrank-symmetric-multi-component-guarantees}
   For any  real $\dimension \times \dimension$ rank-$r$ PSD matrix $\widebar{\A}$
   and arbitrary set ${\Xset \subset \mathbb{R}^{\dimension \times {k}}}$,
   let
   $
   \widebar{\X}_{\star} \eqdef$ $\argmax_{\X \in \Xset}
   	\trace\bigl(\X^{\transpose}{\widebar{\A}}\X\bigr).
   $
   Assuming that there exists an operator 
   ${P_{\Xset}: \mathbb{R}^{\dimension \times {k}} \rightarrow \Xset}$
   such that
   $
   P_{\Xset}(\mathbf{W})\eqdef$ $\argmax_{\X \in \Xset}  \bigl\langle\mathbf{x}_{j},\, \mathbf{w}_{j}\bigr\rangle^{2}
   $,
   then 
   Algorithm~\ref{algo:lowrank-symmetric-multi-component} outputs $\widebar{\X} \in \Xset$ such that
   \begin{align}
   	\trace\bigl(
	 \widebar{\X}^{\transpose} \widebar{\A} \widebar{\X}
	 \bigr)
	  \ge
	  \left(1-\epsilon \right) \cdot
	  \trace\bigl(
	  \widebar{\X}_{\star}^{\transpose} \widebar{\A} \widebar{\X}_{\star}
	  \bigr),
	  \nonumber
   \end{align}
   in time 
   $\Tsvd(r) + O\mathopen{}\bigl(
   		\bigl(\tfrac{4}{\epsilon}\bigr)^{r \cdot {{k}}} \cdot \bigl(T_{\Xset}+{{k}}{\dimension}\bigr)\bigr)$,
		where $T_{\Xset}$ is the time required to compute $P_{\Xset}(\cdot)$ and $\Tsvd(r)$ the time required to compute the truncated SVD of $\widebar{\A}$.
\end{lemma}
\begin{proof}
Let $\widebar{\A} = \widebar{\mathbf{U}}\widebar{\mathbf{\Lambda}}\widebar{\mathbf{U}}^{\transpose}$
denote the truncated eigenvalue decomposition of $\widebar{\A}$;
$\widebar{\mathbf{\Lambda}}$ is a diagonal $r \times r$ whose $i$th diagonal entry $\Lambda_{ii}$ is equal to the $i$th largest eigenvalue of $\widebar{\A}$,
while the columns of $\widebar{\mathbf{U}}$ contain the corresponding eigenvectors. 
By the Cauchy-Schwartz inequality, 
for any $\mathbf{x} \in \mathbb{R}^{\dimension}$,  
   \begin{align}
	  \mathbf{x}^{\transpose}\widebar{\A}\mathbf{x}
	  =
	  \bigl\|
	  	\widebar{\mathbf{\Lambda}}^{1/2}\widebar{\mathbf{U}}^{\transpose} \mathbf{x}
	  \bigr\|_{2}^{2}
	  \ge
	  \bigl\langle
		 \widebar{\mathbf{\Lambda}}^{1/2}\widebar{\mathbf{U}}^{\transpose} \mathbf{x},\,
		 \mathbf{c}
	  \bigr\rangle^{2},
	  \quad
	  \forall\; 
	  \mathbf{c} \in \mathbb{R}^{r}: \|\mathbf{c}\|_{2}=1.
	  \label{CS-on-single-vector}
   \end{align}
   In fact, equality in~\eqref{CS-on-single-vector} is achieved for $\mathbf{c}$ colinear to $\widebar{\mathbf{\Lambda}}^{1/2}\widebar{\mathbf{U}}\mathbf{x}$, and hence,
   \begin{align}
   		\mathbf{x}^{\transpose}\widebar{\A}\mathbf{x}
		=
		\max_{\mathbf{c} \in {\mathbb{S}_{2}^{r-1}}}
		\bigl\langle
		 \widebar{\mathbf{\Lambda}}^{1/2}\widebar{\mathbf{U}}^{\transpose}\mathbf{x},\,
		 \mathbf{c}
	  \bigr\rangle^{2}.
   \end{align}
   In turn,
      \begin{align}
   \trace\left(
		\X^{\transpose}\widebar{\A} \X
   \right)
   =
   \sum_{j=1}^{{k}}
   	{\X^{j}}^{\transpose}\widebar{\A} \X^{j}
	=
	\max_{\mathbf{C}:{\mathbf{C}^{j}} \in {{\mathbb{S}_{2}^{r-1}}} \forall j}
   \sum_{j=1}^{{k}}
   \bigl\langle
		\widebar{\mathbf{\Lambda}}^{1/2}
		\widebar{\mathbf{U}}^{\transpose}
		\X^{j}
		,\,
		\mathbf{C}^{j}
	\bigr\rangle^{2}.
	\label{trace-alt-def}
	\end{align}

	Recall that $\widebar{\X}_{\star}$ is the optimal solution of the trace maximization on $\widebar{\A}$, \textit{i.e.},
   \begin{align}
      \widebar{\X}_{\star}
      \eqdef 
      \argmax_{\X \in \Xset}
		 \trace\left(
		 \X^{\transpose}{\widebar{\A}}\X
		 \right)
		 .
      \nonumber
   \end{align}
   Let $\widebar{\mathbf{C}}_{\star}$ be the maximizing value of $\mathbf{C}$ in~\eqref{trace-alt-def} for $\X=\widebar{\X}_{\star}$, \textit{i.e.},
   $\widebar{\mathbf{C}}_{\star}$ is an $r \times {k}$ matrix with unit-norm columns such that 
   for all $j \in \lbrace 1, \hdots, {k} \rbrace$,
   \begin{align}
	{\widebar{\X}_{\star}^{j}}^{\transpose}
	\widebar{\A}
	{\widebar{\X}_{\star}^{j}}
	=
		\bigl\langle
			\widebar{\mathbf{\Lambda}}^{1/2}
			\widebar{\mathbf{U}}^{\transpose}
			{\widebar{\X}_{\star}^{j}},\,
			{\widebar{\mathbf{C}}_{\star}^{j}}
		\bigr\rangle^{2}.
	\end{align}
Algorithm~\ref{algo:lowrank-symmetric-multi-component} iterates over the points
($r \times {k}$ matrices) $\mathbf{C}$ in~$\mathcal{N}_{\epsilon/2}^{\otimes {k}}\left( {\mathbb{S}_{2}^{r-1}}\right)$, 
the ${k}$th cartesian power of a finite $\sfrac{\epsilon}{2}$-net of the $r$-dimensional $l_{2}$-unit sphere. 
At each such point $\mathbf{C}$, it computes a candidate 
\begin{align}
	\widetilde{\X} 
	=
	\argmax_{\substack{\X \in \Xset }} 
	\sum_{j=1}^{{k}} 
		\bigl\langle \X^{j}, \U\L^{1/2}\bC^{j} \bigr\rangle^{2}
	\nonumber
\end{align}
via Algorithm~\ref{algo:local-candidate}
(See Lemma~\ref{lemma:algo-local-guarantees} for the guarantees of Algorithm~\ref{algo:local-candidate}).
By construction, the set $\mathcal{N}_{\epsilon/2}^{\otimes {k}}\left( {\mathbb{S}_{2}^{r-1}}\right)$ contains a $\mathbf{C}_{\sharp}$
   such that
   \begin{align}
	  \|\mathbf{C}_{\sharp} - \widebar{\mathbf{C}}_{\star}\|_{\infty, 2} 
	  =
	  \max_{j \in \range{{k}}} \|\mathbf{C}_{\sharp}^{j} - \widebar{\mathbf{C}}_{\star}^{j}\|_{2} 
	  \le \epsilon/2.
	 \label{distance-netpoint-from-opt}
   \end{align}
   Based on the above, for all $j \in \lbrace 1, \hdots, {k} \rbrace$,
   \begin{align}
   	\bigl(
		{\widebar{\X}_{\star}^{j}}^{\transpose}
		\widebar{\A}
		{\widebar{\X}_{\star}^{j}}
	\bigr)^{1/2}
	\nonumber & =
	\bigl\lvert 
		\bigl\langle
			\widebar{\mathbf{\Lambda}}^{1/2}
			\widebar{\mathbf{U}}^{\transpose}
			{\widebar{\X}_{\star}^{j}},\,
			{\widebar{\mathbf{C}}_{\star}^{j}}
		\bigr\rangle
	\bigr\rvert
	\nonumber\\ & =
	\bigl\lvert 
		\bigl\langle
			\widebar{\mathbf{\Lambda}}^{1/2}
			\widebar{\mathbf{U}}^{\transpose}
			{\widebar{\X}_{\star}^{j}},\,
			{\mathbf{C}_{\sharp}^{j}}
		\bigr\rangle
		+
		\bigl\langle
			\widebar{\mathbf{\Lambda}}^{1/2}
			\widebar{\mathbf{U}}^{\transpose}
			{\widebar{\X}_{\star}^{j}},\,
			\bigl({\widebar{\mathbf{C}}_{\star}^{j}}-{\mathbf{C}_{\sharp}^{j}}\bigr)
		\bigr\rangle
	\bigr\rvert
	\nonumber\\ & \le
	\bigl\lvert 
		\bigl\langle
			\widebar{\mathbf{\Lambda}}^{1/2}
			\widebar{\mathbf{U}}^{\transpose}
			{\widebar{\X}_{\star}^{j}},\,
			{\mathbf{C}_{\sharp}^{j}}
			\bigr\rangle 
	\bigr\rvert
	+
	\bigl\lvert
		\bigl\langle
			\widebar{\mathbf{\Lambda}}^{1/2}
			\widebar{\mathbf{U}}^{\transpose}
			{\widebar{\X}_{\star}^{j}},\,
			\bigl({\widebar{\mathbf{C}}_{\star}^{j}}-{\mathbf{C}_{\sharp}^{j}} \bigr)
		\bigr\rangle
	\bigr\rvert
	\nonumber\\ & \le
	\bigl\lvert 
		\bigl\langle
			\widebar{\mathbf{\Lambda}}^{1/2}
			\widebar{\mathbf{U}}^{\transpose}
			{\widebar{\X}_{\star}^{j}},\,
			{\mathbf{C}_{\sharp}^{j}}
			\bigr\rangle 
	\bigr\rvert
	+
	\bigl\|
		\widebar{\mathbf{\Lambda}}^{1/2}
		\widebar{\mathbf{U}}^{\transpose}
		{\widebar{\X}_{\star}^{j}}
	\bigr\|
	\cdot 
	\bigl\| {\widebar{\mathbf{C}}_{\star}^{j}}-{\mathbf{C}_{\sharp}^{j}} \bigr\|
	\nonumber\\ & \le
	\bigl\lvert 
		\bigl\langle
			\widebar{\mathbf{\Lambda}}^{1/2}
			\widebar{\mathbf{U}}^{\transpose}
			{\widebar{\X}_{\star}^{j}},\,
			{\mathbf{C}_{\sharp}^{j}}
			\bigr\rangle 
	\bigr\rvert
	+
	(\epsilon/2) \cdot
	\bigl(
		{\widebar{\X}_{\star}^{j}}^{\transpose}
		\widebar{\A}
		{\widebar{\X}_{\star}^{j}}
	\bigr)^{1/2}.
	\label{bounding-opt-quadratic-jthcol-multi}
   \end{align}
   The first step follows by the definition of $\widebar{\mathbf{C}}_{\star}$,
   the second by the linearity of the inner product,
   the third by the triangle inequality,
   the fourth by Cauchy-Schwarz inequality
   and the last by~\eqref{distance-netpoint-from-opt}.
	Rearranging the terms in~\eqref{bounding-opt-quadratic-jthcol-multi},
	\begin{align}
	\bigl\lvert 
		\bigl\langle
			\widebar{\mathbf{\Lambda}}^{1/2}
			\widebar{\mathbf{U}}^{\transpose}
			{\widebar{\X}_{\star}^{j}},\,
			{\mathbf{C}_{\sharp}^{j}}
			\bigr\rangle 
	\bigr\rvert
	\ge
	\bigl( 1 - \tfrac{\epsilon}{2} \bigr)\cdot 
		\bigl(
		{\widebar{\X}_{\star}^{j}}^{\transpose}
		\widebar{\A}
		{\widebar{\X}_{\star}^{j}}
		\bigr)^{1/2}
	\ge
	0,
	\nonumber
   \end{align}
   and in turn,
   \begin{align}
	\bigl\langle
		\widebar{\mathbf{\Lambda}}^{1/2}
		\widebar{\mathbf{U}}^{\transpose}
		{\widebar{\X}_{\star}^{j}},\,
		{\mathbf{C}_{\sharp}^{j}}
	\bigr\rangle^{2}
	\ge
	\bigl( 1 - \tfrac{\epsilon}{2} \bigr)^{2} \cdot 
		{\widebar{\X}_{\star}^{j}}^{\transpose}
		\widebar{\A}
		{\widebar{\X}_{\star}^{j}}
	\ge
	\left( 1 - \epsilon \right) \cdot
		{\widebar{\X}_{\star}^{j}}^{\transpose}
		\widebar{\A}
		{\widebar{\X}_{\star}^{j}}
		\label{inequality-for-each-term}
   \end{align}
   Summing the terms in~\eqref{inequality-for-each-term} over all $j \in \lbrace 1, \hdots, {k}\rbrace$,
   \begin{align}
   \sum_{j=1}^{{k}}
   	\bigl\langle
		\widebar{\mathbf{\Lambda}}^{1/2}
		\widebar{\mathbf{U}}^{\transpose}
		{\widebar{\X}_{\star}^{j}},\,
		{\mathbf{C}_{\sharp}^{j}}
	\bigr\rangle^{2}
	\ge
	\left( 1 - \epsilon \right) \cdot
	\trace\left(\widebar{\X}_{\star}^{\transpose}
		\widebar{\A}
		\widebar{\X}_{\star}
		\right).
	\label{close-to-the-end}
   \end{align}   
   
      Let $\Xsharp \in \Xset$ be the candidate solution produced by the algorithm at $\mathbf{C}_{\sharp}$,
   \textit{i.e.},
   \begin{align}
   	\X_{\sharp}
	\eqdef
	\argmax_{\X \in \Xset}
	\sum_{j=1}^{{k}}
	\bigl\langle\mathbf{x}_{j},\, 
	\widebar{\mathbf{U}}\widebar{\mathbf{\Lambda}}^{1/2}
	{\mathbf{C}_{\sharp}^{j}}\bigr\rangle^{2}.
	\label{optimality-of-x-sharp}
   \end{align} 
  Then,
   \begin{align}
   \trace\left(
		\Xsharp^{\transpose}
		\widebar{\A}
		\Xsharp
   \right)
   \nonumber & \stackrel{(\alpha)}{=}
   \max_{\mathbf{C}:{\mathbf{C}^{j}} \in {{\mathbb{S}_{2}^{r-1}}} \forall j}
   \sum_{j=1}^{{k}}
   \bigl\langle
		\widebar{\mathbf{\Lambda}}^{1/2}
		\widebar{\mathbf{U}}^{\transpose}
		{\widebar{\X}_{\sharp}^{j}}
		,\,
		\mathbf{C}^{j}
	\bigr\rangle^{2}
	\nonumber\\&\stackrel{(\beta)}{\ge}
	  \sum_{j=1}^{{k}}
   \bigl\langle
		\widebar{\mathbf{\Lambda}}^{1/2}
		\widebar{\mathbf{U}}^{\transpose}
		{\widebar{\X}_{\sharp}^{j}}
		,\,
		{\mathbf{C}_{\sharp}^{j}}
	\bigr\rangle^{2}
	\nonumber\\&\stackrel{(\gamma)}{\ge}
	  \sum_{j=1}^{{k}}
   \bigl\langle
		{\widebar{\X}_{\star}^{j}}
		,\,
		\widebar{\mathbf{U}}
		\widebar{\mathbf{\Lambda}}^{1/2}
		{\mathbf{C}_{\sharp}^{j}}
	\bigr\rangle^{2}
	\nonumber\\&\stackrel{(\delta)}{\ge}
	\left( 1 - \epsilon \right) \cdot
	\trace\left(\widebar{\X}_{\star}^{\transpose}
		\widebar{\A}
		\widebar{\X}_{\star}
		\right),
	\label{final-guarantee}
   \end{align}   
   where 
   $(\alpha)$ follows from the observation in~\eqref{trace-alt-def},
   $(\beta)$ from the sub-optimality of  $\mathbf{C}_{\sharp}$,
   $(\gamma)$ by the definition of $\Xsharp$ in~\eqref{optimality-of-x-sharp},
   while 
   $(\delta)$ follows from~\eqref{close-to-the-end}.
   According to~\eqref{final-guarantee},
   at least one of the candidate solutions produced by Algorithm~\ref{algo:lowrank-symmetric-multi-component},
   namely $\Xsharp$,
    achieves an objective value within a multiplicative factor $(1-\epsilon)$ from the optimal, 
    implying the guarantees of the lemma. 
   
   Finally, the running time of Algorithm~\ref{algo:lowrank-symmetric-multi-component} follows immediately from the cost per iteration and the cardinality of the $\sfrac{\epsilon}{2}$-net on the unit-sphere. 
   Note that matrix multiplications can exploit the singular value decomposition which is performed once.
\end{proof}

\begin{reptheorem}{thm:spca-multi-rank-r-guarantees}
 For any  real $\dimension \times \dimension$ rank-$r$ PSD matrix $\widebar{\A}$,
   desired number of components ${k}$,
   number ${s}$ of nonzero entries per component, 
   and accuracy parameter $\epsilon \in (0,1)$,
   Algorithm~\ref{algo:lowrank-symmetric-multi-component}
   outputs $\widebar{\X} \in \Xset_{k}$ such that
   \begin{align}
   	\trace\bigl(
		\widebar{\X}^{\transpose} \widebar{\A} \widebar{\X}
	 \bigr)
	  \;\ge\;
	  \left(1-\epsilon \right) \cdot
	  \trace\bigl(
	  	{\X}_{\star}^{\transpose} \widebar{\A} {\X}_{\star}
	  \bigr),
	  \nonumber
   \end{align}
   where
      $
   {\X}_{\star} \eqdef \argmax_{\X \in \Xset_{k}}
   	\trace\left(\X^{\transpose}{ \widebar{\A} }\X\right),
   $
   in time 
   $\Tsvd(r) + O\mathopen{}\bigl(
   		\bigl(\tfrac{4}{\epsilon}\bigr)^{r \cdot {{k}}} \cdot \dimension \cdot ({s} \cdot {k})^{2} \bigr)$.
$\Tsvd(r)$ is the time required to compute the truncated SVD of $\widebar{\A}$.
\end{reptheorem}
\begin{proof}
	Recall that $\Xset_{{k}}$ is the set of $\dimension \times {k}$ matrices $\X$ whose columns have unit length and pairwise disjoint supports.
	Algorithm~\ref{algo:local-candidate}, given any $\W \in \mathbb{R}^{\dimension \times {k}}$,
	computes $\X \in \Xset_{{k}}$ that optimally solves the constrained maximization in line 5. (See Lemma~\ref{lemma:algo-local-guarantees} for the guarantee of Algorithm~\ref{algo:local-candidate}).
	in time $O\mathopen{}\left(\dimension \cdot ({s} \cdot {k})^{2}\right)$.
	The desired result then follows by Lemma~\ref{lowrank-symmetric-multi-component-guarantees} for the constrained set $\Xset_{{k}}$.
\end{proof}

\subsection{Guarantees of Algorithm~\ref{algo:symmetric-multi-component} -- Proof of Theorem~\ref{thm:generic-sketch-solution}}
\label{sec:proof-algo-general-matrix}
We prove Theorem~\ref{thm:generic-sketch-solution} with the approximation guarantees of Algorithm~\ref{algo:symmetric-multi-component}.
\begin{lemma}
   \label{lemma:generic-sketch-solution}
   For any $\dimension \times \dimension$ PSD matrices $\mathbf{A}$ and $\widebar{\A}$,
   and any set $\Xset \subseteq \mathbb{R}^{\dimension \times {{k}}}$ let
   $$
   	\X_{\star} \eqdef \argmax_{\X \in \Xset}\trace\left(\X^{\transpose}\mathbf{A}\X \right),
   	\quad \text{and} \quad
	\widebar{\X}_{\star} \eqdef \argmax_{\X \in \Xset}\trace\bigl(\X^{\transpose}{\widebar{\A}}\X\bigr).
	$$
   Then, for any $\widebar{\X} \in \Xset$ such that
   $
   \trace\bigl(\widebar{\X}^{\transpose}{\widebar{\A}}\widebar{\X}\bigr)
   \ge \gamma \cdot
   \trace\bigl(\widebar{\X}_{\star}^{\transpose}{\widebar{\A}}\widebar{\X}_{\star}\bigr)
   $
   for some $0 < \gamma < 1$, 
   \begin{align}
      \trace\bigl(\widebar{\X}^{\transpose}\mathbf{A}\widebar{\X}\bigr)
      \ge
      \gamma \cdot \trace\bigl(\X_{\star}^{\transpose}\mathbf{A}\X_{\star}\bigr)
      - 2 \cdot \|\mathbf{A}-{\widebar{\A}}\|_{2} \cdot \max_{\X \in \Xset} \|\X\|_{\frob}^{2}.
      \nonumber
   \end{align}
\end{lemma}
\begin{proof}
By the optimality of $\widebar{\X}_{\star}$ for ${\widebar{\A}}$,
\begin{align}
   \trace\left(\widebar{\X}_{\star}^{\transpose}{\widebar{\A}}\widebar{\X}_{\star}\right)
   \ge
   \trace\left(\X_{\star}^{\transpose}{\widebar{\A}}\X_{\star}\right).
   \nonumber
\end{align}
In turn,
for any $\widebar{\X} \in \Xset$ such that
 $
   \trace\left(\widebar{\X}^{\transpose}{\widebar{\A}}\widebar{\X}\right)
   \ge \gamma \cdot
   \trace\left(\widebar{\X}_{\star}^{\transpose}{\widebar{\A}}\widebar{\X}_{\star}\right)
 $
 for some $0 < \gamma < 1$,
\begin{align}
   \trace\left(\widebar{\X}^{\transpose}{\widebar{\A}}\widebar{\X}\right)
   \ge
   \gamma \cdot
   \trace\left(\X_{\star}^{\transpose}{\widebar{\A}}\X_{\star}\right).
   \label{eq:bp01}
\end{align}
Let $\mathbf{E} \eqdef \mathbf{A} - \widebar{\A}$.
By the linearity of the trace,
\begin{align}
   \trace\left(\widebar{\X}^{\transpose}{\widebar{\A}}\widebar{\X}\right)
   & =
   \trace\left(\widebar{\X}^{\transpose}\mathbf{A}\widebar{\X}\right)
   -
   \trace\left(\widebar{\X}^{\transpose}\mathbf{E}\widebar{\X}\right)
   \nonumber\\ & \le
   \trace\left(\widebar{\X}^{\transpose}\mathbf{A}\widebar{\X}\right)
   +
   \bigl\lvert\trace\left(\widebar{\X}^{\transpose} \mathbf{E} \widebar{\X}\right)\bigr\rvert.
   \label{eq:bp02}
\end{align}
By Lemma~\ref{lemma:abs-trace-XAY-ub},
\begin{align}
   \bigl\lvert\trace\left(\widebar{\X}^{\transpose}\mathbf{E}\widebar{\X}\right)\bigr\rvert
   \le
   \|\widebar{\X}\|_{\frob} \cdot \|\widebar{\X}\|_{\frob} \cdot \|\mathbf{E}\|_{2}
   \le
   \| \mathbf{E} \|_{2}
   \cdot \max_{\X \in \Xset} \|\X\|_{\frob}^{2}
   \;\eqdef\; R.
   \label{residual-trace-bound}
\end{align}
Continuing from~\eqref{eq:bp02},
\begin{align}
   \trace\left(\widebar{\X}^{\transpose}{\widebar{\A}}\widebar{\X}\right)
   \le
   \trace\left(\widebar{\X}^{\transpose}\mathbf{A}\widebar{\X}\right)
   +
   R.
   \label{eq:bp04}
\end{align}
Similarly,
\begin{align}
   \trace\left(\X_{\star}^{\transpose}{\widebar{\A}}\X_{\star}\right)
   &=
   \trace\left(\X_{\star}^{\transpose}\mathbf{A}\X_{\star}\right)
   -
   \trace\left(\X_{\star}^{\transpose}\mathbf{E}\X_{\star}\right)
   \nonumber\\& \ge
   \trace\left(\X_{\star}^{\transpose}\mathbf{A}\X_{\star}\right)
   -
   \bigl\lvert
   \trace\left(\X_{\star}^{\transpose}\mathbf{E}\X_{\star}\right)
   \bigr\rvert
   \nonumber\\& \ge
   \trace\left(\X_{\star}^{\transpose}\mathbf{A}\X_{\star}\right)
   -
   R.
   \label{eq:bp03}
\end{align}
Combining the above, we have
\begin{align}
   \trace\left(\widebar{\X}^{\transpose}\mathbf{A}\widebar{\X}\right)
   &\ge
   \trace\left(\widebar{\X}^{\transpose}{\widebar{\A}}\widebar{\X}\right)
   -R
   \nonumber\\&\ge
   \gamma \cdot
   \trace\left(\X_{\star}^{\transpose}{\widebar{\A}}\X_{\star}\right)
   -
   R
   \nonumber\\&\ge
   \gamma \cdot
   \left(
   \trace\left(\X_{\star}^{\transpose}\mathbf{A}\X_{\star}\right)
   -
   R
   \right) - R
   \nonumber\\&=
   \gamma \cdot \trace\left(\X_{\star}^{\transpose}\mathbf{A}\X_{\star}\right)
   - (1 + \gamma) \cdot  R
   \nonumber\\&\ge
   \gamma \cdot \trace\left(\X_{\star}^{\transpose}\mathbf{A}\X_{\star}\right)
   - 2 \cdot  R,
   \nonumber
\end{align}
where
the first inequality follows from~\eqref{eq:bp04}
the second from~\eqref{eq:bp01},
the third from~\eqref{eq:bp03}, and the last from the fact that $R \ge 0$ and $0< \gamma \le 1$.
This concludes the proof.
\end{proof}


\begin{remark}
   If in Lemma~\ref{lemma:generic-sketch-solution} the PSD matrices
   $\mathbf{A}$ and $\widebar{\A} \in \mathbb{R}^{\dimension \times \dimension}$
   are such that $\mathbf{A}-\widebar{\A}$ is also PSD, then the following tighter bound holds:
   \begin{align}
      \trace\bigl(\widebar{\X}^{\transpose}\mathbf{A}\widebar{\X} \bigr)
      \ge
      \gamma \cdot \trace\bigl(\X_{\star}^{\transpose}\mathbf{A}\X_{\star} \bigr)
      - \sum_{i=1}^{{k}} \lambda_{i} \bigl(\mathbf{A}-{\widebar{\A}}\bigr).
      \nonumber
   \end{align}
\end{remark}
\begin{proof}
 This follows from the fact that
 if $\mathbf{E}\eqdef \mathbf{A}-\widebar{\A}$ is PSD,
 then
 \begin{align}
   \trace\left(\widebar{\X}^{\transpose}\mathbf{E}\widebar{\X}\right)
   =
   \sum_{j=1}^{\dimension} \mathbf{x}_{j}^{\transpose}\mathbf{E}\mathbf{x}_{j}
   \ge 0,
   \nonumber
 \end{align}
 and the bound in~\eqref{eq:bp02} can be improved to
 \begin{align}
   \trace\left(\widebar{\X}^{\transpose}{\widebar{\A}}\widebar{\X}\right)
   & =
   \trace\left(\widebar{\X}^{\transpose}\mathbf{A}\widebar{\X}\right)
   -
   \trace\left(\widebar{\X}^{\transpose}\mathbf{E}\widebar{\X}\right)
    \le
   \trace\left(\widebar{\X}^{\transpose}\mathbf{A}\widebar{\X}\right).
   \nonumber
\end{align}
Further, by Lemma~\ref{cor:orthogonal-proj-bound}, the bound in~\eqref{residual-trace-bound} can be improved to
\begin{align}
   \trace\bigl(\widebar{\X}^{\transpose}\mathbf{E}\widebar{\X}\bigr)
   \le
   \sum_{i=1}^{{k}} \lambda_{i} \bigl(\mathbf{E}\bigr)
   \;\eqdef\; R.
   \nonumber
\end{align}
The rest of the proof follows as is.
\end{proof}


\begin{reptheorem}{thm:generic-sketch-solution}
	For any $n \times \dimension$ input data matrix $\mathbf{S}$,
	 with corresponding empirical covariance matrix $\A=\sfrac{1}{n}\cdot \mathbf{S}^{\T}\mathbf{S}$,
	 any desired number of components~$k$, 
	and accuracy parameters~$\epsilon \in (0,1)$ and~$r$,
	Algorithm~\ref{algo:symmetric-multi-component} outputs ${\X_{\mathsmaller{(r)}} \in \Xset_{k}}$ such that
\begin{align}
	\trace\bigl( {\X}_{\mathsmaller{(r)}}^{\T} \A {\X}_{\mathsmaller{(r)}} \bigr) 
	\;\ge\;
	(1 - \epsilon) \cdot 
	\trace\bigl( \X_{\star}^{\T} \A \X_{\star} \bigr)
	-
	2 \cdot {k} \cdot \|\A - \widebar{\A}\|_{2},
   \nonumber
\end{align}
where
$\Xstar \eqdef \argmax_{\X \in \Xset_{k}} \trace\left(\X^{\T} \A \X \right)$,
in time
$\Tsketch(r) + \Tsvd(r) + O\mathopen{}\bigl(\bigl(\tfrac{4}{\epsilon}\bigr)^{r \cdot {{k}}} \cdot {\dimension} \cdot ({s} \cdot {k})^{2}\bigr)$.
\end{reptheorem}
\begin{proof}
	The theorem follows from Lemma~\ref{lemma:generic-sketch-solution}
	and the approximation guarantees of Algorithm~\ref{algo:lowrank-symmetric-multi-component}.
\end{proof}

\subsection{Proof of Theorem \ref{thm:opt-epsk}}\label{proof:algo-ptas}

First, we restate and prove the following Lemma by \cite{alon2013approximate}.

\begin{lemma}{\label{thm:Ad-approx}}
Let $\A \in \R^{d \times d}$ be an positive semidefinite matrix with entries in $[-1,1]$, and ${\bf V} \in \R^{d \times d}$ matrix such that $\A = {\bf V}{\bf V}^\top$. Consider a random matrix ${\bf R} \in \R^{d \times r}$ with entries drawn according to a Gaussian distribution $N(0,1/r)$, and define $$\widebar{\A} = {\bf V}{\bf R}{\bf R}^\top{\bf V}^\top.$$ Then, for  $r = O(\epsilon^{-2}\log d)$,
\begin{align}
	\left|[\A]_{i,j}-[\widebar{\A}]_{i,j}\right|\le \epsilon \nonumber
\end{align} for all $i,j$ with probability at least $1-1/d$.
\end{lemma}

\begin{proof}
The proof relies on the Johnson-Lindenstrauss (JL) Lemma \cite{dasgupta2003elementary},
according to which for any two unit norm vectors $\x,\y\in\mathbb{R}^d$ and $\mathbf{R}$ generated as described
\begin{equation}
	\Pr\bigl\{|\x^\T{\bf R}{\bf R}^\T\y - \x^\T\y|\ge \epsilon\bigr\}\le 2\cdot e^{-(\epsilon^2-\epsilon^3)\cdot r/4}.\nonumber
\end{equation}
Observe that each element of $\A$ is in $[-1,1]$, hence can be rewritten as an inner product of two unit-norm vectors:
$$[\A]_{i,j}={\bf V}_{:,i}^T{\bf V}_{:,j}.$$
Setting $r = O(\epsilon^{-2}\log d)$
and using the JL lemma and a union bound over all $O(d^2)$ vector pairs ${\bf V}_{:,i}$, ${\bf V}_{:,j}$ 
we obtain the desired result. 
\end{proof}

Next, we provide the proof of Theorem \ref{thm:opt-epsk} for the simple case of $k = 1$; the proof easily generalizes to the multi-component case $k > 1$. According to Lemma~\ref{thm:Ad-approx}, choosing $d = O\left((\delta/6)^{-2} \log{n}\right)= O\left(\delta^{-2} \log{n}\right)$ suffices for all entries of $\widebar{\A}$ constructed as described in the lemma to satisfiy
\begin{align}
   \left|[\A]_{i,j}-[\widebar{\A}]_{i,j}\right|
   \le 
   \frac{\delta}{6} \nonumber
\end{align}
with probability at least $1-1/d$. In turn, for any $s$-sparse, unit-norm $\x$,
\begin{align}
	\bigl\vert{\bf x}^\T\A{\bf x}-{\bf x}^\T\widebar{\A}{\bf x}\bigr\vert 
&= \left\vert \sum_{i,j} x_i x_j ([\A]_{ij}-[\widebar{\A}]_{ij}) \right\vert  
   \le \frac{\delta}{6} \cdot \left|\sum_{i=1}^n |x_i| \sum_{j=1}^n |x_j| \right| \nonumber \\ 
&\le \frac{\delta}{6} \cdot \|x\|_1^2 
   \le\frac{\delta}{6}\cdot \left(\sqrt{s} \cdot \|\x\|_2\right)^2 
    = \frac{\delta}{6} \cdot s
	\label{eq:A-Ad-quadratic}, 
\end{align}
where the second inequality follows from the fact that $\x$ is $s$-sparse and unit norm.

We run Algorithm~\ref{algo:lowrank-symmetric-multi-component} (for $k=1$) with input argument the rank-$r$ matrix $\widebar{\A}$, desired sparsity~${s}$ and accuracy parameter $\epsilon = \delta/6$.
Algorithm~\ref{algo:lowrank-symmetric-multi-component} outputs a ${s}$-sparse unit-norm vector $\widehat{\mathbf{x}}$
which according to Theorem~\ref{thm:spca-multi-rank-r-guarantees} satisfies
\begin{align}
      \left(1-\sfrac{\delta}{6}\right) \cdot {\mathbf{x}_d}^{\transpose} \widebar{\A}{\mathbf{x}_d}
   \;\le\;
      \widehat{\mathbf{x}}^{\transpose} \widebar{\A} \widehat{\mathbf{x}}
   \;\le\;
      {\mathbf{x}_d}^{\transpose} \widebar{\A}{\mathbf{x}_d},
\end{align}
where $\mathbf{x}_d$ is the true $s$-sparse principal component of $\widebar{\A}$.
This, in turn, implies that $\widehat{\mathbf{x}}$ satisfies
\begin{align}
   \left| 
      \widehat{\mathbf{x}}^{\transpose} \widebar{\A} \widehat{\mathbf{x}}
      -
      {\mathbf{x}_d}^{\transpose} \widebar{\A}{\mathbf{x}_d}
   \right|
   \le
      \frac{\delta}{6} \cdot
      {\mathbf{x}_d}^{\transpose} \widebar{\A}{\mathbf{x}_d} 
   \le
      \frac{\delta}{6}
      \left( 1 + \frac{\delta}{6}\right)s \le \frac{\delta}{3} \cdot {s},
      \label{eq:by-theor-1}
\end{align}
where the second inequality follows from the fact that the entries of $\widebar{\A}$ lie in $[-1-\frac{\delta}{6}, 1+\frac{\delta}{6}]$ and $\widehat{\mathbf{x}}$ is ${s}$-sparse and unit-norm.

In the following, 
we bound the difference of the performance of $\widehat{\mathbf{x}}$ on the original matrix $\mathbf{A}$ from the optimal value.
Let $\x_{\star}$ denote the ${s}$-sparse principal component of $\A$ and define
\begin{align}
	\opt
	\eqdef
	{\x_{\star}}^T \mathbf{A} \x_{\star}. \nonumber
\end{align}
Then,
\begin{align}
	| \opt  - \widehat{\bf x}^\T\A\widehat{\bf x} |
	&=| \x_{\star}^\T\A\x_{\star} - 
		 \widehat{\bf x}^\T\A\widehat{\bf x} | 	\nonumber\\
	&=| \x_{\star}^\T\A\x_{\star} 
	       - \x_{d}^\T\widebar{\A}\x_{d}+ {\x_{d}}^\T\widebar{\A}\x_{d} 
	       -\widehat{\bf x}^\T\A\widehat{\bf x} |  \nonumber \\
	&\le\underbrace{| \x_{\star}^\T\A\x_{\star} -{\x_{d}}^\T\widebar{\A}\x_{d}|}_{A}
	       + \underbrace{| {\x_{d}}^\T\widebar{\A}\x_{d} - \widehat{\bf x}^\T\A\widehat{\bf x} |}_{B}.
	\label{eq:prwto-part}
\end{align}
Utilizing \eqref{eq:A-Ad-quadratic} and the triangle inequality,
one can verify that
\begin{align}
A
   &=| {\x_{\star}}^\T\A\x_{\star} - {\x_{d}}^\T\A\x_{d} \phantom{|}+ \phantom{|}{\x_{d}}^\T\A\x_{d}- {\x_{d}}^\T\widebar{\A}\x_{d}	|	 \nonumber \\
   &\le | {\x_{\star}}^\T\A\x_{\star} - {\x_{d}}^\T\A\x_{d} | + 
      |{\x_{d}}^\T\A\x_{d}- {\x_{d}}^\T\widebar{\A}\x_{d}	|	 \nonumber \\
   &\le
      \phantom{|}
	 \underbrace{ {\x_{\star}}^\T\A\x_{\star} - {\x_{d}}^\T\A\x_{d} }_{\ge 0}\phantom{|}
	 + \phantom{|}\frac{\delta}{6} \cdot{s} \phantom{|}\nonumber\\
   &\le
      \phantom{|}{\x_{\star}}^\T\A\x_{\star} - {\x_{d}}^\T\A\x_{d} \phantom{|}
      +
       \phantom{|}\frac{\delta}{6} \cdot{s} \phantom{|}
       +
      \phantom{|}\underbrace{ {\x_{d}}^\T\widebar{\A}\x_{d} - {\x_{\star}}^\T\widebar{\A}\x_{\star}}_{\ge 0}\phantom{|}
     \nonumber\\
   &\le
      \phantom{|}{\x_{\star}}^\T\A\x_{\star} - {\x_{\star}}^\T\widebar{\A}\x_{\star}\phantom{|}
       + \phantom{|}\frac{\delta}{6} \cdot{s} \phantom{|}
      + \phantom{|}{\x_{d}}^\T\widebar{\A}\x_{d} - {\x_{d}}^\T\A\x_{d}\phantom{|}
     \nonumber\\
   &\le
      | {\x_{\star}}^\T\A\x_{\star} - {\x_{\star}}^\T\widebar{\A}\x_{\star} |
       + \phantom{|}\frac{\delta}{6} \cdot{s} \phantom{|}
      + | {\x_{d}}^\T\widebar{\A}\x_{d} - {\x_{d}}^\T\A\x_{d} |
      \nonumber\\
   & \le 
   \frac{\delta}{2}\cdot{s}.
   \label{eq:deftero-part}
\end{align}
Similarly,
\begin{align}
B
   &=\left| {\x_{d}}^\T\widebar{\A}\x_{d} - \widehat{\bf x}^\T\widebar{\A}\widehat{\bf x}
	    + \widehat{\bf x}^\T\widebar{\A}\widehat{\bf x} - \widehat{\bf x}^\T\A\widehat{\bf x} \right|  \nonumber \\
   &=\left| {\x_{d}}^\T\widebar{\A}\x_{d} - \widehat{\bf x}^\T\widebar{\A}\widehat{\bf x} \right|
	    + \left|\widehat{\bf x}^\T\widebar{\A}\widehat{\bf x} - \widehat{\bf x}^\T\A\widehat{\bf x} \right|  \nonumber \\
   &\le
   \left| {\x_{d}}^\T\widebar{\A}\x_{d} - \widehat{\bf x}^\T\widebar{\A}\widehat{\bf x} \right|
	    + \frac{\delta}{6} \cdot{s}  \nonumber \\
   &\stackrel{(\alpha)}{\le}
      \frac{2\delta}{6} \cdot{s} + \frac{\delta}{6} \cdot{s}  \nonumber \\
   &\stackrel{}{\le} \frac{\delta}{2}\cdot{s}.
   \label{eq:trito-part}
\end{align}
where 
$(\alpha)$ follows from \eqref{eq:by-theor-1}.
Continuing from \eqref{eq:prwto-part},
combining \eqref{eq:deftero-part} and \eqref{eq:trito-part} we obtain
\begin{align}
	\left| \opt \right. & \left.- \widehat{\bf x}^\T\A\widehat{\bf x} \right|
	\le \delta \cdot{s},
	\nonumber
\end{align}
which is the desired result.
\hfill$\blacksquare$

\section{Auxiliary Technical Lemmata}

\begin{lemma}
   \label{lemma:partial-sing-val-sum}
   For any real $\dimension \times n$ matrix $\mathbf{M}$,
   and any $r, k \le \min\lbrace \dimension, n \rbrace$,
   \begin{align}
	  \sum_{i=r+1}^{r + k} \sigma_{i}(\mathbf{M})
	  \le
	  \frac{{k}}{\sqrt{r+k}} \cdot \|\mathbf{M}\|_{\frob},
	  \nonumber
   \end{align} 
   where $\sigma_{i}(\mathbf{M})$ is the $i$th largest singular value of $\mathbf{M}$.
\end{lemma}
\begin{proof}
   By the Cauchy-Schwartz inequality,
   \begin{align}
	  \sum_{i=r+1}^{r + k} \sigma_{i}(\mathbf{M})
	  =
	  \sum_{i=r+1}^{r + k} \lvert\sigma_{i}(\mathbf{M})\rvert
	  \le
	  \left( \sum_{i=r+1}^{r + k} \sigma_{i}^{2}(\mathbf{M}) \right)^{1/2} \cdot \|\mathbf{1}_{{k}}\|_{2}
	  =
	  \sqrt{{k}} \cdot \left( \sum_{i=r+1}^{r + k} \sigma_{i}^{2}(\mathbf{M}) \right)^{1/2}.
	  \nonumber
   \end{align}
   Note that $\sigma_{r+1}(\mathbf{M}), \hdots, \sigma_{r+k}(\mathbf{M})$ are the ${k}$ smallest among the $r+k$ largest singular values. 
   Hence,
   \begin{align}
	  \sum_{i=r+1}^{r + k} \sigma_{i}^{2}(\mathbf{M})
	  \le
	  \frac{{k}}{r+k}\sum_{i=1}^{r + k} \sigma_{i}^{2}(\mathbf{M})
	  \le
	  \frac{{k}}{r+k} \sum_{i=1}^{\min\lbrace{\dimension,n}\rbrace} \sigma_{i}^{2}(\mathbf{M})
	  =
	  \frac{{k}}{r+k} \|\mathbf{M}\|_{\frob}^{2}.
	  \nonumber
   \end{align}
   Combining the two inequalities, the desired result follows.
\end{proof}
\begin{corollary}
	\label{sigma-bound}
   For any real $\dimension \times n$ matrix $\mathbf{M}$
   and $k \le \min\lbrace \dimension, n \rbrace$,
   $\sigma_{{k}}(\mathbf{M}) \le  k^{-1/2} \cdot \| \mathbf{M} \|_{\frob}$.
\end{corollary}
\begin{proof}
   It follows immediately from Lemma~\ref{lemma:partial-sing-val-sum}.
\end{proof}
\begin{lemma}
   \label{holder-consequence}
   Let $a_1,\hdots , a_n$ 
   and $b_1, \hdots, b_n$ be $2n$ real numbers and let $p$ and
   $q$ be two numbers such that ${1/p} + {1/q} = 1$ and $p>1$. We have
   \begin{align}
	  \bigl\lvert 
		 \sum_{i=1}^{n} a_{i}b_{i}
	  \bigr\rvert
	  \le
	  \left( \sum_{i=1}^{n} \lvert  a_{i}\rvert^{p} \right)^{1/p}
	  \cdot
	  \left( \sum_{i=1}^{n} \lvert  b_{i}\rvert^{q} \right)^{1/q}.
	  \nonumber
   \end{align}
\end{lemma}
\begin{lemma}
   \label{lemma:frob-of-matrix-prod}
   For any two real matrices $\mathbf{A}$ and $\mathbf{B}$ of appropriate dimensions,
   \begin{align}
	  \|\mathbf{A}\mathbf{B}\|_{\frob}
	  \le
	  \min\mathopen{}
	  \left\lbrace
		 \|\mathbf{A}\|_{2} \|\mathbf{B}\|_{\frob}, \;
		 \|\mathbf{A}\|_{\frob} \|\mathbf{B}\|_{2}
	  \right\rbrace.
	  \nonumber
   \end{align}
\end{lemma}
\begin{proof}
   Let $\mathbf{b}_{i}$ denote the $i$th column of $\mathbf{B}$. 
   Then,
   \begin{align}
	  \|\mathbf{A}\mathbf{B}\|_{\frob}^{2}
	  =
	  \sum_{i} \| \mathbf{A} \mathbf{b}_{i} \|_{2}^{2}
	  \le
	  \sum_{i} \| \mathbf{A}\|_{2}^{2}  \|\mathbf{b}_{i} \|_{2}^{2}
	  =
	  \| \mathbf{A}\|_{2}^{2} \sum_{i}   \|\mathbf{b}_{i} \|_{2}^{2}
	  =
	  \| \mathbf{A}\|_{2}^{2} \|\mathbf{B} \|_{\frob}^{2}.
	  \nonumber 
   \end{align}
   Similarly, using the previous inequality,
   \begin{align}
	  \|\mathbf{A}\mathbf{B}\|_{\frob}^{2}
	  =
	  \|\mathbf{B}^{\transpose}\mathbf{A}^{\transpose}\|_{\frob}^{2}
	  \le
	  \|\mathbf{B}^{\transpose}\|_{2}^{2}\|\mathbf{A}^{\transpose}\|_{\frob}^{2}
	  =
	  \|\mathbf{B}\|_{2}^{2}\|\mathbf{A}\|_{\frob}^{2}.
	  \nonumber
   \end{align}
   Combining the two upper bounds, the desired result follows.
\end{proof}
\begin{lemma}
   \label{lemma:inner-prod-ub}
   For any $\mathbf{A}, \mathbf{B} \in \mathbb{R}^{n \times {{k}}}$,
   \begin{align}
	  \bigl\lvert 
		 \langle \mathbf{A}, \mathbf{B} \rangle
	  \bigr\rvert 
	  \eqdef 
	  \bigl\lvert
		 \trace\left( \mathbf{A}^{\transpose} \mathbf{B}\right)
	  \bigr\rvert
	  \le
	  \|\mathbf{A}\|_{\frob}
	  \|\mathbf{B}\|_{\frob}.
	  \nonumber
   \end{align}
\end{lemma}
\begin{proof} 
	The inequality follows from Lemma~\ref{holder-consequence} for ${p=q=2}$, treating $\mathbf{A}$ and $\mathbf{B}$ as vectors.
\end{proof}
%

\begin{lemma}
	\label{lemma:orthogonal-proj-bound}
   For any real
   $m \times n$ matrix $\mathbf{A}$,
   and any $k \le \min\lbrace {m},\; {n}\rbrace$, 
      \begin{align}
   		\max_{
			\substack{
				{\mathbf{Y}} \in \mathbb{R}^{n \times {{k}}}\\
				{\mathbf{Y}}^{\transpose}{\mathbf{Y}} = \mathbf{I}_{{k}} 
			}
		}
   		\|\mathbf{A} {\mathbf{Y}} \|_{\frob}
		=
		\left( 		\sum_{i=1}^{{k}} \sigma_{i}^{2}(\mathbf{A}) \right)^{1/2}.
		\nonumber
   \end{align}
   The maximum is attained by $\mathbf{Y}$ coinciding with the ${k}$ leading right singular vectors of $\mathbf{A}$.
\end{lemma}
\begin{proof}
	Let $\mathbf{U}\mathbf{\Sigma}\mathbf{V}^{\transpose}$ be the singular value decomposition of $\mathbf{A}$;
	$\mathbf{U}$ and $\mathbf{V}$ are $m \times m$ and $n \times n$ unitary matrices respectively, while $\Sigma$ is a diagonal matrix with ${\Sigma_{jj} = \sigma_{j}}$, the $j$th largest singular value of~$\mathbf{A}$, $j=1, \hdots, d$, where $d \eqdef \min\lbrace {m}, {n} \rbrace$.
	Due to the invariance of the Frobenius norm under unitary multiplication,
   \begin{align}
   		\|\mathbf{A} \mathbf{Y} \|_{\frob}^{2}
		=
		\|\mathbf{U}\mathbf{\Sigma}\mathbf{V}^{\transpose} \mathbf{Y} \|_{\frob}^{2}
		=
		\|\mathbf{\Sigma}\mathbf{V}^{\transpose} \mathbf{Y} \|_{\frob}^{2}.
		\label{frob-norm-unitary}
   \end{align}
   Continuing from~\eqref{frob-norm-unitary},
   \begin{align}
		\|\mathbf{\Sigma}\mathbf{V}^{\transpose} \mathbf{Y} \|_{\frob}^{2}
		=
		\trace\left(\mathbf{Y}^{\transpose}\mathbf{V}\mathbf{\Sigma}^{2}\mathbf{V}^{\transpose} \mathbf{Y}\right)
		=
		\sum_{i=1}^{{k}} \mathbf{y}_{i}^{\transpose} 
			\Bigl( 
				\sum_{j=1}^{d} \sigma_{j}^{2} \cdot \mathbf{v}_{j} \mathbf{v}_{j}^{\transpose}
			\Bigr)
			\mathbf{y}_{i}
		=
		\sum_{j=1}^{d} 
			\sigma_{j}^{2} \cdot \sum_{i=1}^{{k}} \left( \mathbf{v}_{j}^{\transpose} \mathbf{y}_{i}\right)^{2}.
		\nonumber
   \end{align}
   Let $z_{j} \eqdef \sum_{i=1}^{{k}} \left( \mathbf{v}_{j}^{\transpose} \mathbf{y}_{i}\right)^{2}$, $j=1, \hdots, d$.
   Note that each individual $z_{j}$ satisfies
   \begin{align}
   		 0 \le z_{j} \eqdef \sum_{i=1}^{{k}} \left( \mathbf{v}_{j}^{\transpose} \mathbf{y}_{i}\right)^{2}
		\le
		\|\mathbf{v}_{j}\|^{2} = 1,
		\nonumber
   \end{align}
   where the last inequality follows from the fact that the columns of $\mathbf{Y}$ are orthonormal.
   Further,
   \begin{align}
   		\sum_{j=1}^{d} z_{j} 
		=  
		\sum_{j=1}^{d}\sum_{i=1}^{{k}} \left( \mathbf{v}_{j}^{\transpose} \mathbf{y}_{i}\right)^{2}
		=
		\sum_{i=1}^{{k}}\sum_{j=1}^{d} \left( \mathbf{v}_{j}^{\transpose} \mathbf{y}_{i}\right)^{2}
		=
		\sum_{i=1}^{{k}}\|\mathbf{y}_{i}\|^{2} = k.
		\nonumber
   \end{align}
   Combining the above, we conclude that
   \begin{align}
		\|\mathbf{A}\mathbf{Y} \|_{\frob}^{2}
		=
		\sum_{j=1}^{d} 
			\sigma_{j}^{2} \cdot z_{j}
		\le \sigma_{1}^{2} + \hdots + \sigma_{{k}}^{2}.
		\label{frob-prod-ub}
   \end{align}
   Finally, it is straightforward to verify that if $\mathbf{y}_{i} = \mathbf{v}_{i}$, $i=1, \hdots, {{k}}$, then~\eqref{frob-prod-ub} holds with equality.
\end{proof}

\begin{lemma}
   \label{lemma:abs-trace-XAY-ub}
   For any real
   $\dimension \times n$ matrix $\mathbf{A}$,
   and pair of 
   $\dimension \times {{k}}$ matrix $\X$ and  $n \times {{k}}$ matrix $\Y$
   such that $\X^{\transpose}\X=\mathbf{I}_{{k}}$ and $\Y^{\transpose}\Y=\mathbf{I}_{{k}}$
   with ${k} \le \min\lbrace {\dimension},\; {n}\rbrace$, the following holds:
   \begin{align}
	  \bigl\lvert
	  \trace\left( \X^{\transpose} \mathbf{A} \Y \right)
	  \bigr\rvert
	  \le
	  \sqrt{{k}}\cdot 
	  \Bigl( \sum_{i=1}^{{k}} \sigma_{i}^{2}\left(\mathbf{A}\right) \Bigr)^{1/2}.
	  \nonumber
   \end{align}
\end{lemma}
\begin{proof}
   By Lemma~\ref{lemma:inner-prod-ub},
   \begin{align}
	  \lvert \langle \X,\, \mathbf{A}\Y \rangle\rvert
	  =
	  \bigl\lvert
	  \trace\left( \X^{\transpose} \mathbf{A} \Y \right)
	  \bigr\rvert
	  \le
	  \|\X\|_{\frob} \cdot \|\mathbf{A}\Y\|_{\frob}
	  =
	  \sqrt{{k}} \cdot \|\mathbf{A}\Y\|_{\frob}.
	  \nonumber
   \end{align}
   where the last inequality follows from the fact that $\|\X\|_{\frob}^{2} = \trace\left(\X^{\transpose}\X\right) = \trace\left(\mathbf{I}_{{k}}\right) = {k}$.
Combining with a bound on $\|\mathbf{A} \Y \|_{\frob}$ as in Lemma~\ref{lemma:orthogonal-proj-bound}, completes the proof.
\end{proof}
\begin{lemma}
	\label{cor:orthogonal-proj-bound}
   For any real $\dimension \times \dimension$ PSD matrix $\mathbf{A}$,
   and $k \times \dimension$ matrix $\mathbf{X}$ with $k \le \dimension$ orthonormal columns,
      \begin{align}
   		\trace\left(\X^{\transpose}\mathbf{A}\X\right)
		\le
		\sum_{i=1}^{{k}} \lambda_{i}(\mathbf{A})
		\nonumber
   \end{align}
   where $\lambda_{i}(\mathbf{A})$ is the $i$th largest eigenvalue of $\mathbf{A}$.
   Equality is achieved for $\X$ coinciding with the ${k}$ leading eigenvectors of~$\mathbf{A}$.
\end{lemma}
\begin{proof}
	Let $\mathbf{A} = \mathbf{V}\mathbf{V}^{\transpose}$ be a factorization of the PSD matrix $\mathbf{A}$.
	Then, $\trace\left(\X^{\transpose}\mathbf{A}\X\right) = \trace\left(\X^{\transpose}\mathbf{V}\mathbf{V}^{\transpose}\X\right) = \|\mathbf{V}^{\transpose}\X\|_{\frob}^{2}$.
	The desired result follows by Lemma~\ref{lemma:orthogonal-proj-bound}
	and the fact that $\lambda_{i}(\mathbf{A}) = \sigma_{i}^{2}(\mathbf{V})$, $i=1,\hdots, \dimension$.
\end{proof}

\section{Additional Experimental Results}
\label{sec:apndx-experiments}

\begin{figure}[tbph!]
	\centering
   \subfigure[tight][]{
	  \includegraphics[height=0.38\textwidth, trim=0cm 1pt 0cm 0cm, clip=true]{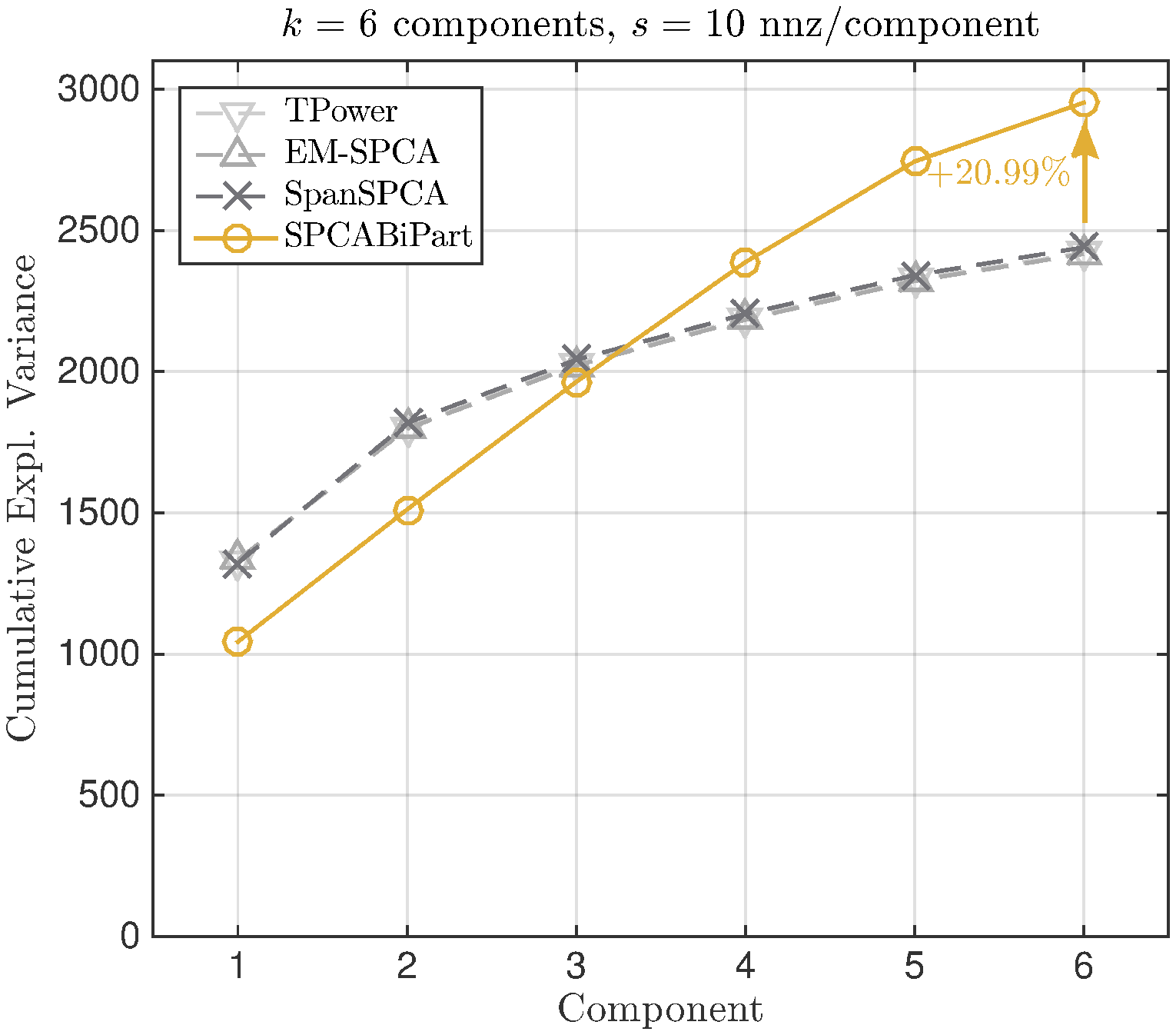}
	  \label{cvar-line-nips-s10}
   }
   \subfigure[tight][]{
	  \includegraphics[height=0.38\textwidth, trim=0cm 1pt 0cm 0cm, clip=true]{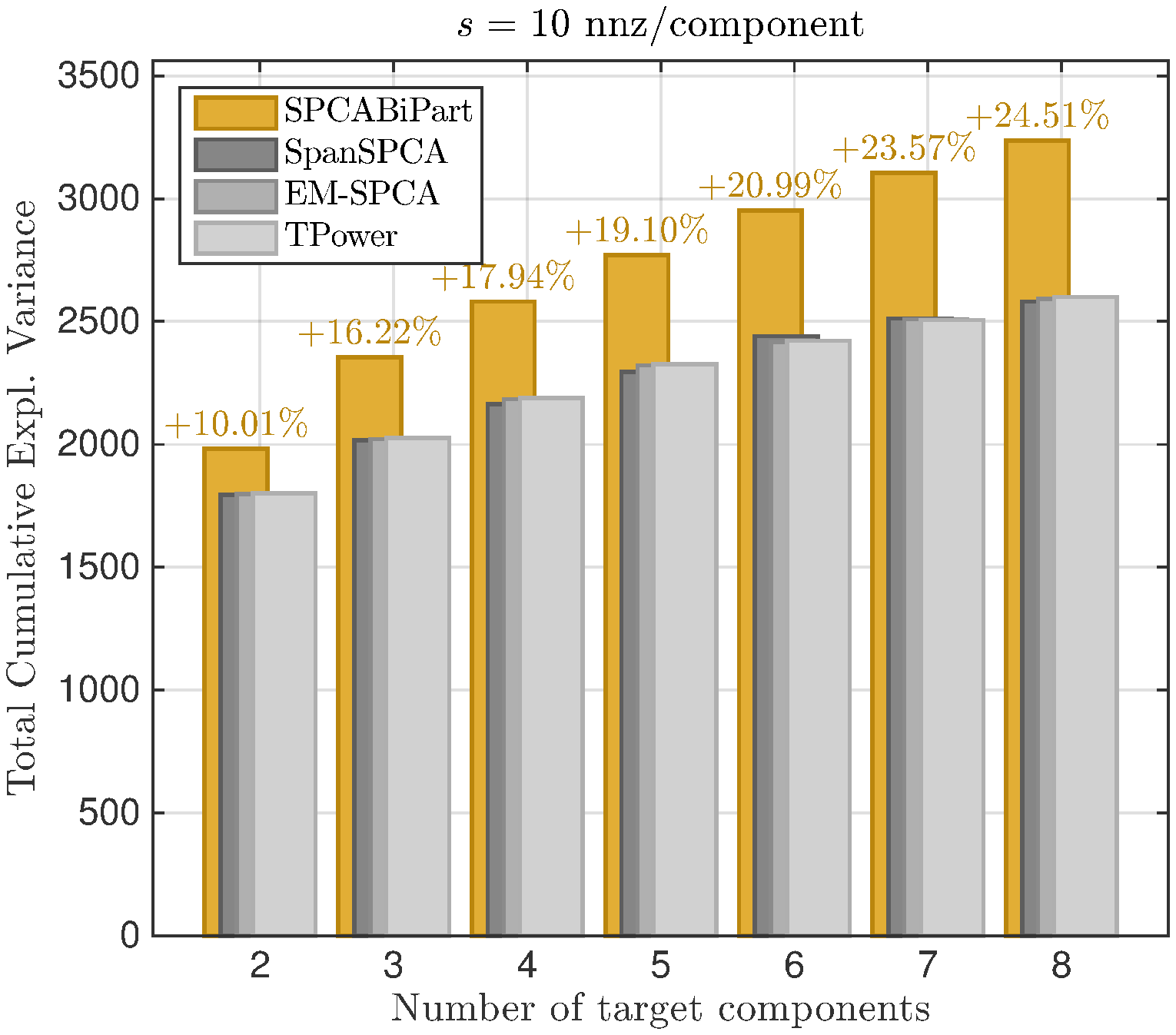}
	  \label{cvar-barplot-nips-s10}
   }
   \vspace{-1.2em}
\caption{
	Cumulative variance captured by $k$ $s$-spars components computed on the word-by-word matrix -- \textsc{BagOfWords:NIPS} dataset~\cite{Lichman:2013}.
	Sparsity is arbitrarily set to $s=10$ nonzero entries per component. 
	Fig.~\ref{cvar-line-nips-s10} depicts the cum. variance captured by ${k=6}$ components. 
	Deflation leads to a greedy formation of components; first components capture high variance, but subsequent ones contribute less. 
	On the contrary, our algorithm jointly optimizes the ${k}$ components and achieves higher total cum. variance. 
	Fig.~\ref{cvar-barplot-nips-s10} depicts the total cum. variance achieved for various values of~$k$. 
	Our algorithm operates on a rank-$4$ approximation of the input.
	}
	\label{bow-nips-plots-k8-s10}
\end{figure}

 \renewcommand{\arraystretch}{1.1}
\begin{table}[tbph!]
	\setcounter{magicrownumbers}{0}
	\scriptsize
	\centering
	\setlength{\tabcolsep}{0.9\tabcolsep}
	\rowcolors{2}{white}{mustard!20!white}
	\begin{tabular}{%
		>{\hspace{-\tabcolsep plus .5em}{\makebox[1.3em][r]{\rownumber\space}}\hspace{.3em}}l<{}%
		>{}l<{}%
		>{}l<{}%
		>{}l<{}%
		>{}l<{}%
		>{}l<{}%
		>{}l<{}%
		>{}l<{\hspace{-\tabcolsep}}%
	}
	\toprulec
	Topic $1$ &
	Topic $2$ &
	Topic $3$ &
	Topic $4$ &
	Topic $5$ &
	Topic $6$ &
	Topic $7$ &
	Topic $8$ \\
	\midrulec
	\midrulec
    \llap{\smash{\rotatebox[origin=c]{90}{%
      \hspace*{-9\normalbaselineskip}
      \textsc{TPower}}}%
      \hspace*{2.3em}}
	network&        algorithm&      neuron& parameter&      object& classifier&     word&   noise\\
model&  data&   cell&   point&  image&  net&    speech& control\\
learning&       system& pattern&        distribution&   recognition&    classification& level&    dynamic\\
input&  error&  layer&  hidden& images& class&  context&        step\\
function&       weight& information&    space&  task&   test&   hmm&    term\\
neural& problem&        signal& gaussian&       features&       order&  character&      optimal\\
unit&   result& visual& linear& feature&        examples&       processing&     component\\
set&    number& field&  probability&    representation& rate&   non&    equation\\
training&       method& synaptic&       mean&   performance&    values& approach&       single\\
output& vector& firing& case&   view&   experiment&     trained&        analysis\\
\midrulec
    \llap{\smash{\rotatebox[origin=c]{90}{%
      \hspace*{-9.5\normalbaselineskip}
      \textsc{SpanSPCA}}}%
      \hspace*{2.3em}}
network&        algorithm&      neuron& parameter&      recognition&    control&        classifier&       noise\\
model&  data&   cell&   distribution&   object& action& classification& order\\
input&  weight& pattern&        point&  image&  dynamic&        class&  term\\
learning&       error&  layer&  linear& word&   step&   net&    component\\
neural& problem&        signal& probability&    performance&    optimal&        test&   rate\\
function&       output& information&    space&  task&   policy& speech& equation\\
unit&   result& visual& gaussian&       features&       states& examples&       single\\
set&    number& synaptic&       hidden& representation& reinforcement&  approach&       analysis\\
system& method& field&  case&   feature&        values& experiment&     large\\
training&       vector& response&       mean&   images& controller&     trained&        form\\
\midrulec
    \llap{\smash{\rotatebox[origin=c]{90}{%
      \hspace*{-10\normalbaselineskip}
      \textsc{SPCABiPart}}}%
      \hspace*{2.3em}}
data&   function&       neuron& unit&   learning&       network&        model&  training\\
distribution&   algorithm&      cell&   weight& space&  input&  parameter&      hidden\\
gaussian&       set&    visual& layer&  action& neural& information&    performance\\
probability&    error&  direction&      net&    order&  system& control&        recognition\\
component&      problem&        firing& task&   step&   output& dynamic&        classifier\\
approach&       result& synaptic&       connection&     linear& pattern&        mean&   test\\
analysis&       number& response&       activation&     case&   signal& noise&  word\\
mixture&        method& spike&  architecture&   values& processing&     field&  speech\\
likelihood&     vector& activity&       generalization& term&   image&  local&  classification\\
experiment&     point&  motion& threshold&      optimal&        object& equation&       trained\\
	\bottomrulec
   \end{tabular}
   \vspace{1em}
   
   \rowcolors{1}{white}{white}
   \begin{tabular}{lc}
	  \toprulec
 					 & Total Cum. Variance \\
	  \midrulec
        \textsc{TPower} & $2.5999 \cdot 10^3$ \\
        \textsc{SpanSPCA} & $2.5981 \cdot 10^3$ \\
        \textsc{SPCABiPart} & $\mathbf{3.2090 \cdot 10^3}$ \\
        \bottomrulec
   \end{tabular}

   \caption{
  \textsc{ BagOfWords:NIPS} dataset~\cite{Lichman:2013}.
   We run various SPCA algorithms for $k=8$ components (topics)
   and ${s=10}$ nonzero entries per component. 
   The table lists the words selected by each component (words corresponding to higher magnitude entries appear higher in the topic).
   Our algorithm was configured to use a rank-$4$ approximation of the input data.
   }
   \label{}
\end{table}   

\begin{figure}[!ht]
	\centering
   \subfigure[tight][]{
	  \includegraphics[height=0.38\textwidth, trim=0cm 1pt 0cm 0cm, clip=true]{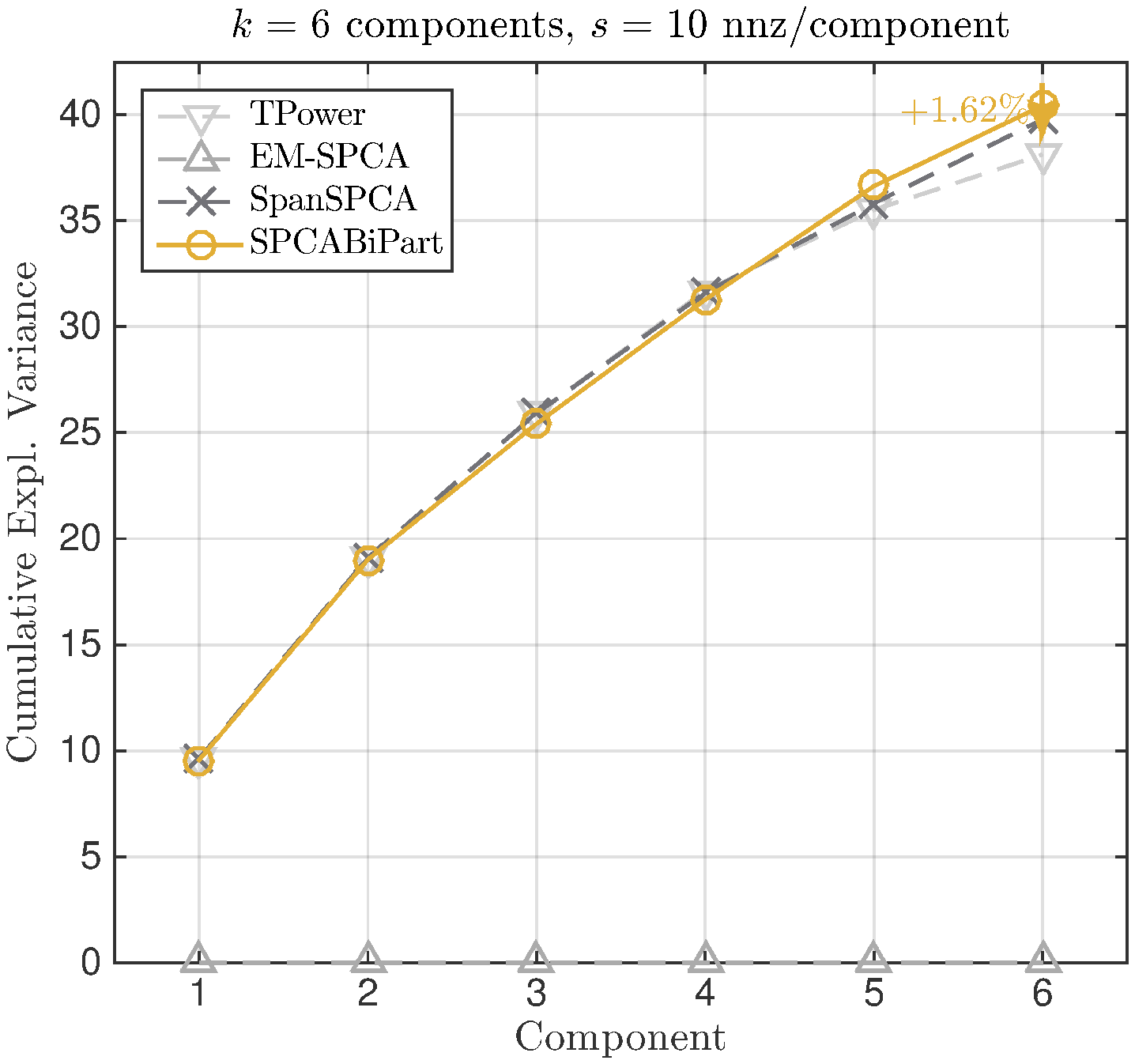}
	  \label{cvar-line-nytimes-s10}
   }
   \subfigure[tight][]{
	  \includegraphics[height=0.38\textwidth, trim=0cm 1pt 0cm 0cm, clip=true]{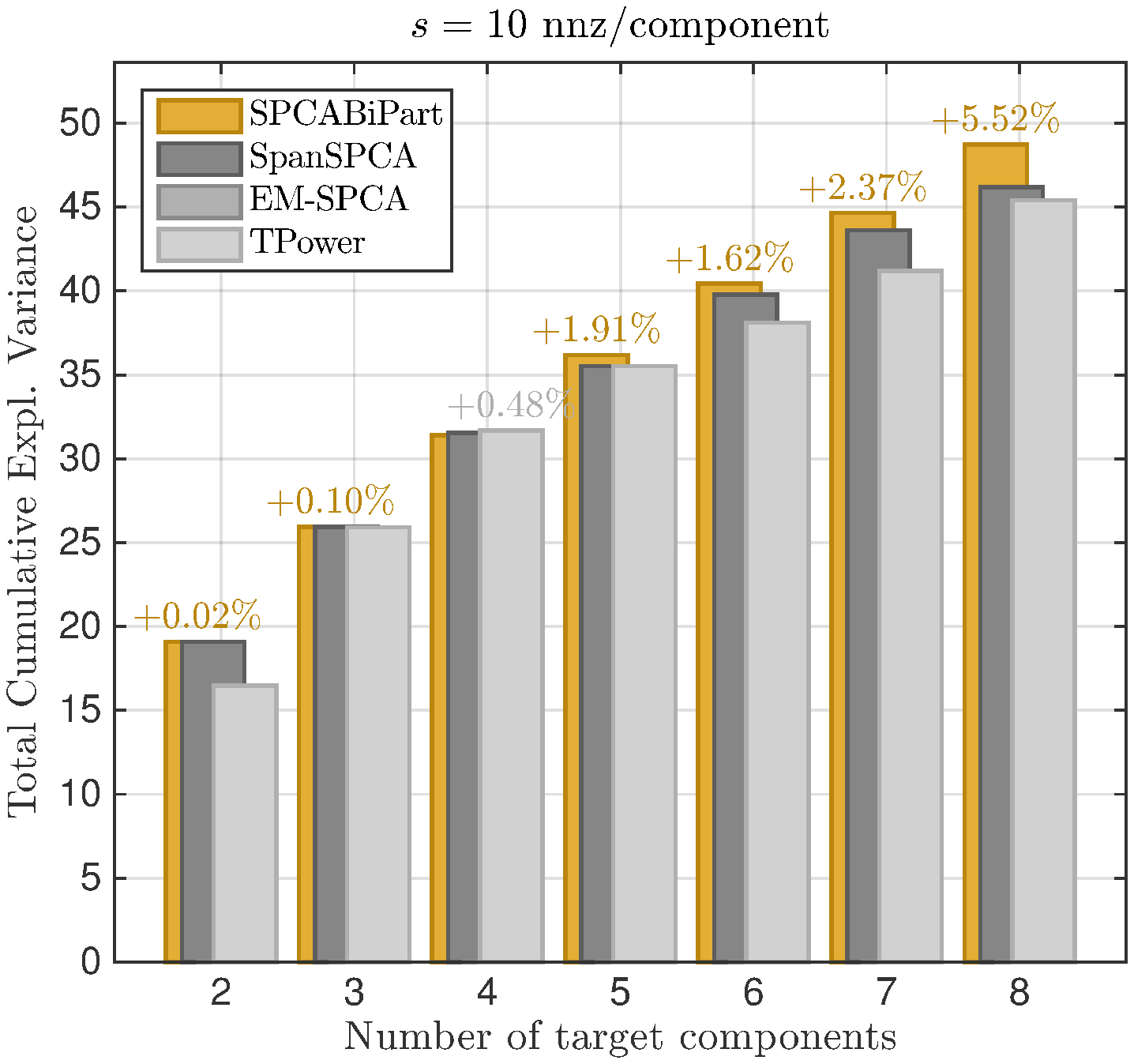}
	  \label{cvar-barplot-nytimes-s10}
   }
   \vspace{-1.2em}
	\caption{
	   	Cumulative variance captured by $k$ $s$-sparse ($s=10$) extracted components on the word-by-word matrix -- \textsc{BagOfWords:NyTimes} dataset~\cite{Lichman:2013}.
		Fig.~\ref{cvar-line-nytimes-s10} depicts the cum. variance captured by ${k=6}$ components. 
Deflation leads to a greedy formation of components;
first components capture high variance, but subsequent ones contribute less.
On the contrary, our algorithm jointly optimizes the ${k}$ components and achieves higher total cum. variance. 
	Fig.~\ref{cvar-barplot-nytimes-s10} depicts the total cum. variance achieved for various values of $k$. 
	Sparsity is arbitrarily set to $s=10$ nonzero entries per component. 
	Our algorithm operates on a rank-$4$ approximation.
	}
	\label{bow:nytimes-plots-k8-s10}
\end{figure}

\begin{figure}[!ht]
	\centering
   \subfigure[tight][]{
	  \includegraphics[height=0.38\textwidth, trim=0cm 1pt 0cm 0cm, clip=true]{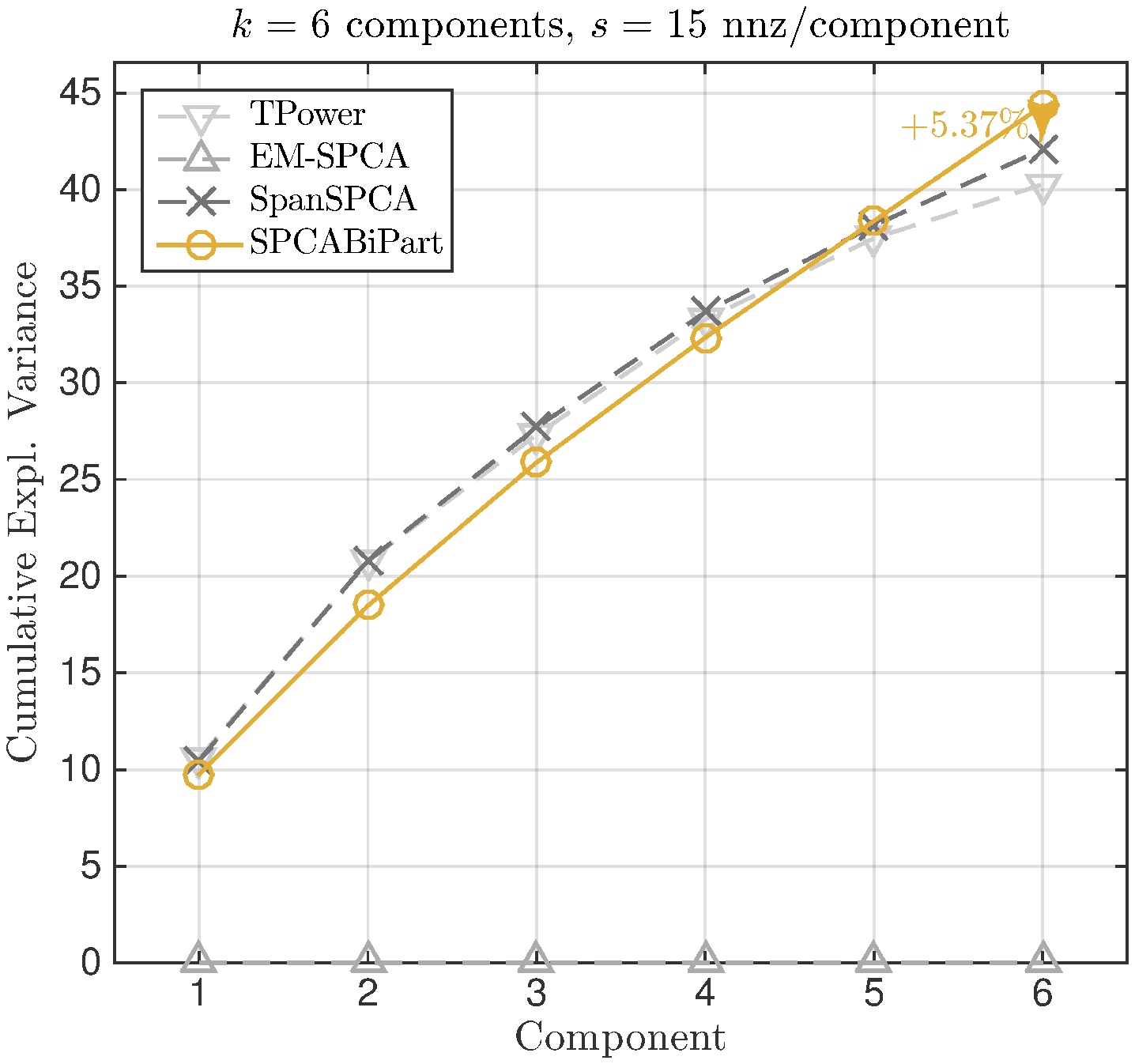}
	  \label{cvar-line-nytimes-s15}
   }
   \subfigure[tight][]{
	  \includegraphics[height=0.38\textwidth, trim=0cm 1pt 0cm 0cm, clip=true]{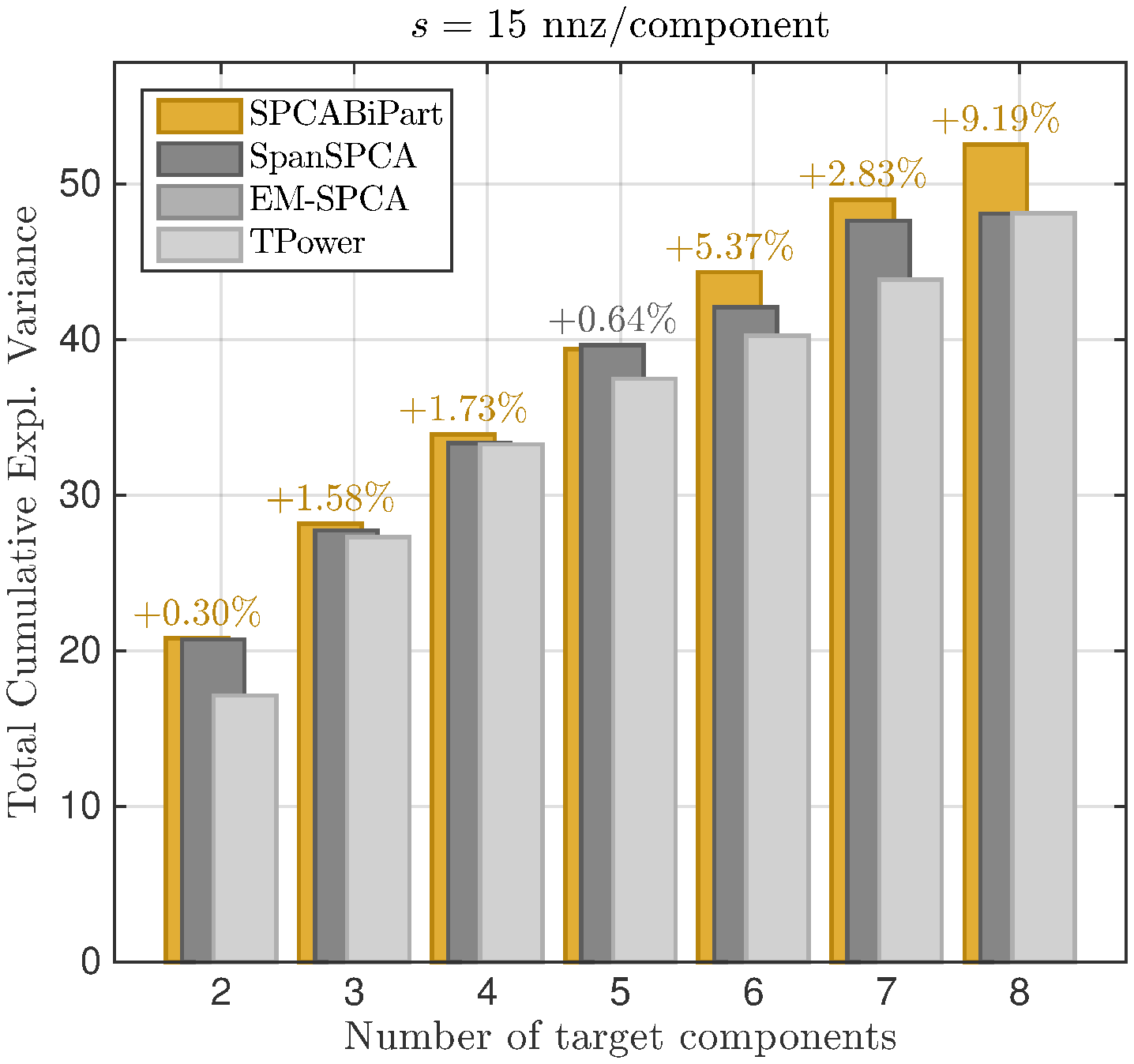}
	  \label{cvar-barplot-nytimes-s15}
   }
   \vspace{-1.2em}
	\caption{
	Same as Fig.~\ref{bow:nytimes-plots-k8-s10}, but for sparsity $s=15$.
	}
	\label{bow:nytimes-plots-k8-s15}
\end{figure}

\begin{figure}[!ht]
	\centering
   \subfigure[tight][]{
	   \includegraphics[height=0.38\textwidth, trim=0cm 1pt 0cm 0cm, clip=true]{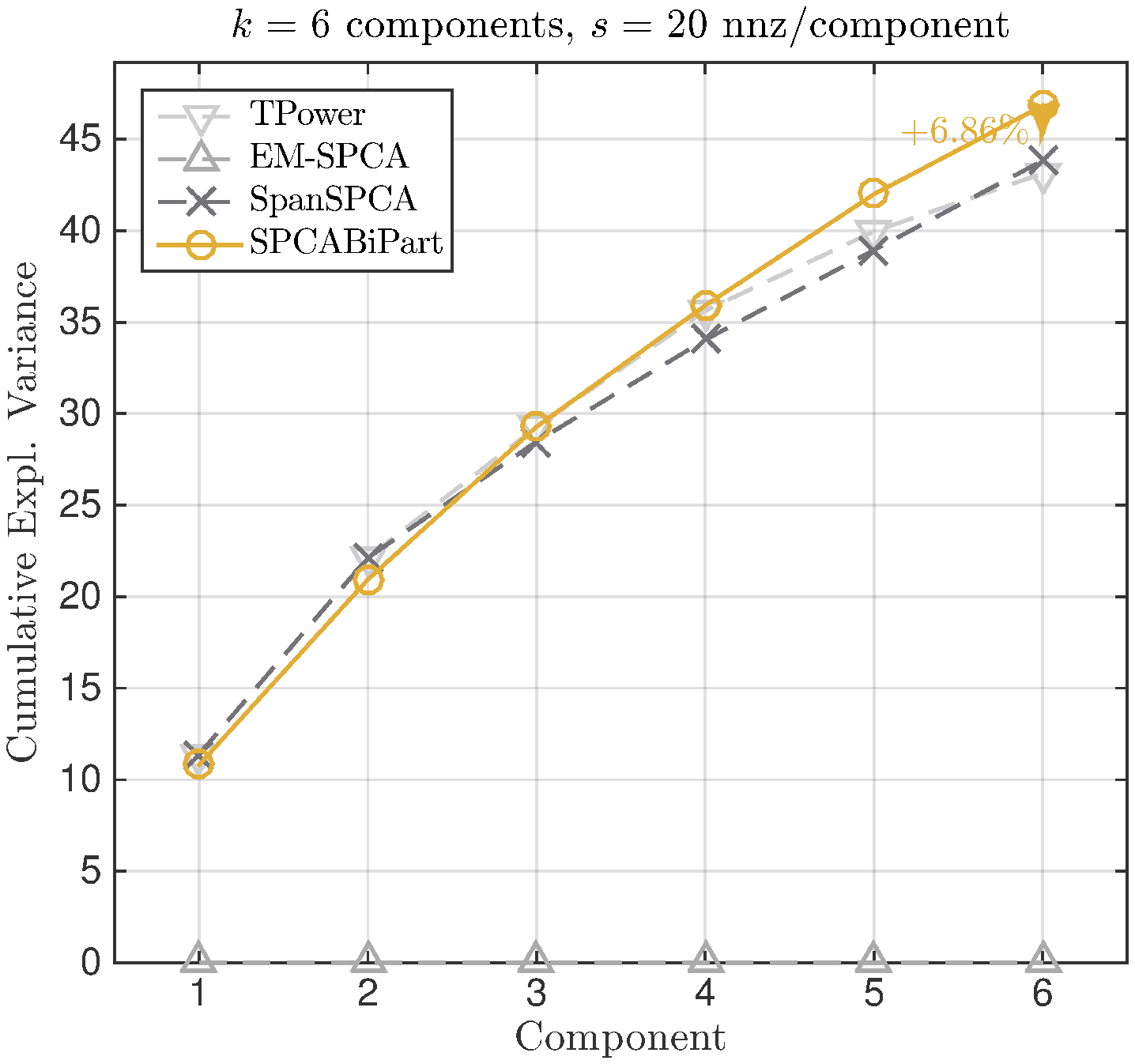}
\label{cvar-line-nytimes-s20}
 }
   \subfigure[tight][]{
   \includegraphics[height=0.38\textwidth, trim=0cm 1pt 0cm 0cm, clip=true]{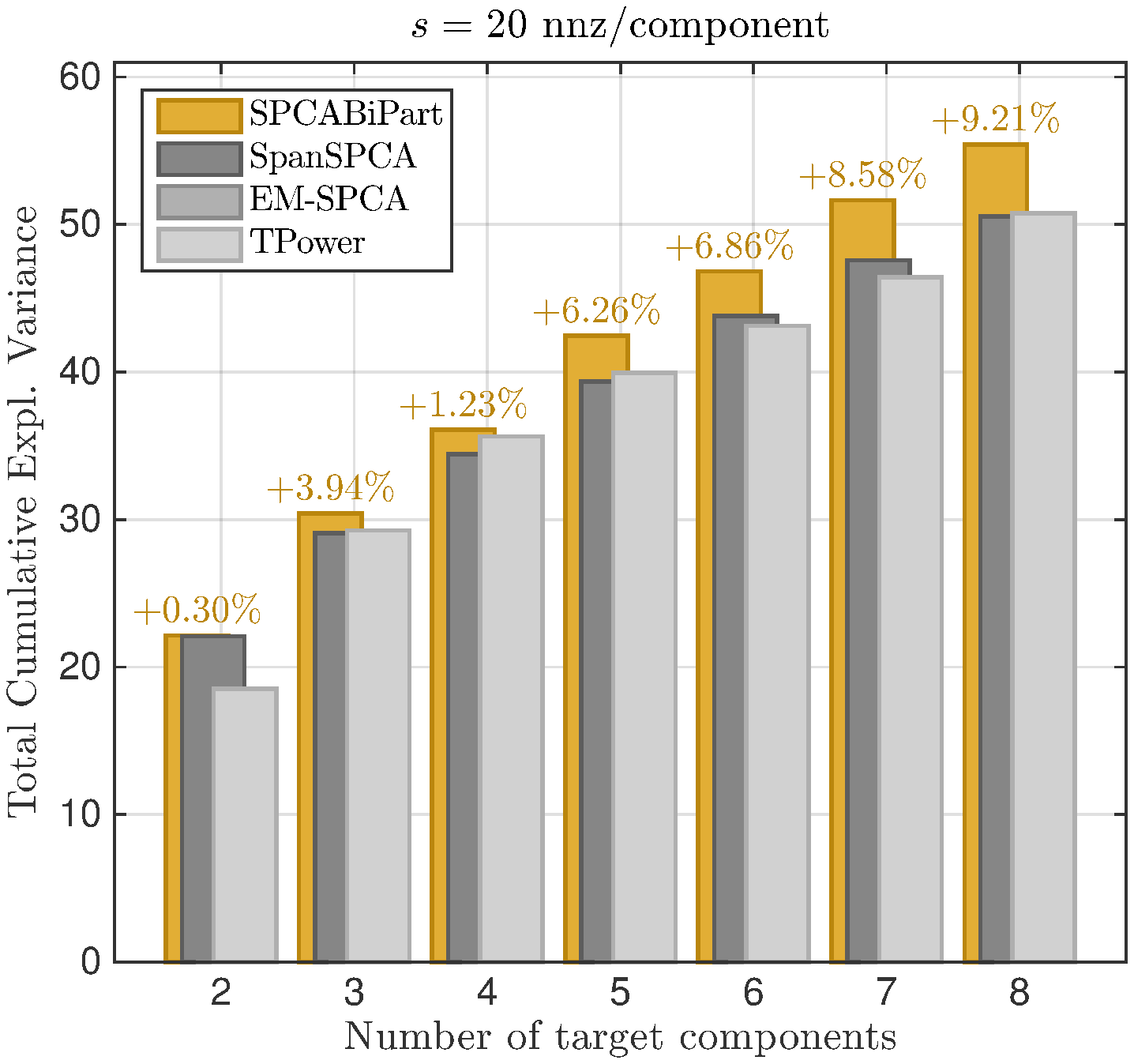}
	  \label{cvar-barplot-nytimes-s20}
   }
   \vspace{-1.2em}
	\caption{
	Same as Fig.~\ref{bow:nytimes-plots-k8-s10}, but for sparsity $s=20$.
	}
	\label{bow:nytimes-plots-k8-s20}
\end{figure}

\renewcommand{\arraystretch}{1.1}
\begin{table}[tbph!]
	\setcounter{magicrownumbers}{0}
	\fontsize{7.0}{8}\selectfont
	\setlength{\tabcolsep}{0.85\tabcolsep}
	\centering
	\rowcolors{2}{white}{mustard!20!white}
	\noindent\makebox[\textwidth]{%
	\begin{tabular}{%
		>{\hspace{-\tabcolsep plus .5em}{\makebox[1.3em][r]{\rownumber\space}}\hspace{.3em}}l<{}%
		>{}l<{}%
		>{}l<{}%
		>{}l<{}%
		>{}l<{}%
		>{}l<{}%
		>{}l<{}%
		>{}l<{\hspace{-\tabcolsep}}%
	}
	\toprulec
	Topic $1$ &
	Topic $2$ &
	Topic $3$ &
	Topic $4$ &
	Topic $5$ &
	Topic $6$ &
	Topic $7$ &
	Topic $8$ \\
	\midrulec
	\midrulec
	\llap{\smash{\rotatebox[origin=c]{90}{%
      \hspace*{-9\normalbaselineskip}
      \textsc{TPower}}}%
      \hspace*{2.3em}}
percent&        zzz\_bush&      team&   school& women&  zzz\_enron&     drug&   palestinian\\
company&        zzz\_al\_gore&  game&   student&        show&   firm&   patient&        zzz\_israel\\
million&        president&      season& program&        book&   zzz\_arthur\_andersen&  doctor& zzz\_israeli\\
companies&      official&       player& high&   com&    deal&   system& zzz\_yasser\_arafat\\
market& zzz\_george\_bush&      play&   children&       look&   lay&    problem&        attack\\
stock&  campaign&       games&  right&  american&       financial&      law&    leader\\
business&       government&     point&  group&  need&   energy& care&   peace\\
money&  plan&   run&    home&   part&   executives&     cost&   israelis\\
billion&        administration& coach&  public& family& accounting&     help&   israeli\\
fund&   zzz\_white\_house&      win&    teacher&        found&  partnership&    health& zzz\_west\_bank\\
	\midrulec
	\llap{\smash{\rotatebox[origin=c]{90}{%
      \hspace*{-9\normalbaselineskip}
      \textsc{SpanSPCA}}}%
      \hspace*{2.3em}}
percent&        team&   zzz\_bush&      palestinian&    school& cup&    show&   won\\
company&        game&   zzz\_al\_gore&  attack& student&        minutes&        com&    night\\
million&        season& president&      zzz\_united\_states&    children&       add&    part&   left\\
companies&      player& zzz\_george\_bush&      zzz\_u\_s&      program&        tablespoon&     look&   big\\
market& play&   campaign&       military&       home&   teaspoon&       need&   put\\
stock&  games&  official&       leader& family& oil&    book&   win\\
business&       point&  government&     zzz\_israel&    women&  pepper& called& hit\\
money&  run&    political&      zzz\_american&  public& water&  hour&   job\\
billion&        right&  election&       war&    high&   large&  american&       ago\\
plan&   coach&  group&  country&        law&    sugar&  help&   zzz\_new\_york\\
	\midrulec
	\llap{\smash{\rotatebox[origin=c]{90}{%
      \hspace*{-9\normalbaselineskip}
      \textsc{SPCABiPart}}}%
      \hspace*{2.3em}}
percent&        zzz\_united\_states&    zzz\_bush&      company&        team&   cup&    school& zzz\_al\_gore\\
million&        zzz\_u\_s&      official&       companies&      game&   minutes&        student&        zzz\_george\_bush\\
money&  zzz\_american&  government&     market& season& add&    children&       campaign\\
high&   attack& president&      stock&  player& tablespoon&     women&  election\\
program&        military&       group&  business&       play&   oil&    show&   plan\\
number& palestinian&    leader& billion&        point&  teaspoon&       book&   tax\\
need&   war&    country&        analyst&        run&    water&  family& public\\
part&   administration& political&      firm&   right&  pepper& look&   zzz\_washington\\
problem&        zzz\_white\_house&      american&       sales&  home&   large&  hour&   member\\
com&    games&  law&    cost&   won&    food&   small&  nation\\
	\bottomrulec
   \end{tabular}%
   }
   \vspace{1em}
   
   \rowcolors{1}{white}{white}
   \begin{tabular}{lc}
	  \toprulec
 					 & Total Cum. Variance \\
	  \midrulec
        \textsc{TPower} & $45.4014$ \\
        \textsc{SpanSPCA} & $46.0075$ \\
        \textsc{SPCABiPart} & $\mathbf{47.7212}$ \\
        \bottomrulec
   \end{tabular}
   
   \caption{
  \textsc{ BagOfWords:NyTimes} dataset~\cite{Lichman:2013}.
   We run various SPCA algorithms for $k=8$ components (topics)
   and ${s=10}$ nonzero entries per component. 
   The table lists the words selected by each component (words corresponding to higher magnitude entries appear higher in the topic).
   Our algorithm was configured to use a rank-$4$ approximation of the input data.
   }
   \label{bow:nytimes:table-topics-3algos-k8-s10}
\end{table}   

 \renewcommand{\arraystretch}{1.1}
\begin{table}[tbph!]
	\setcounter{magicrownumbers}{0}
	\fontsize{7.0}{8}\selectfont
	\setlength{\tabcolsep}{0.85\tabcolsep}
	\centering
	\rowcolors{2}{white}{mustard!20!white}
	
	\noindent\makebox[\textwidth]{%
	\begin{tabular}{%
		>{\hspace{-\tabcolsep plus .5em}{\makebox[1.3em][r]{\rownumber\space}}\hspace{.3em}}l<{}%
		>{}l<{}%
		>{}l<{}%
		>{}l<{}%
		>{}l<{}%
		>{}l<{}%
		>{}l<{}%
		>{}l<{\hspace{-\tabcolsep}}%
	}
	\toprulec
	Topic $1$ &
	Topic $2$ &
	Topic $3$ &
	Topic $4$ &
	Topic $5$ &
	Topic $6$ &
	Topic $7$ &
	Topic $8$ \\
	\midrulec
	\midrulec
	\llap{\smash{\rotatebox[origin=c]{90}{%
      \hspace*{-14\normalbaselineskip}
      \textsc{TPower}}}%
      \hspace*{2.3em}}
	percent&	zzz\_bush&	team&	school&	com&	zzz\_enron&	law&	palestinian\\
company&	zzz\_al\_gore&	game&	student&	women&	firm&	drug&	zzz\_israel\\
million&	zzz\_george\_bush&	season&	program&	book&	deal&	court&	zzz\_israeli\\
companies&	campaign&	player&	children&	web&	financial&	case&	zzz\_yasser\_arafat\\
market&	right&	play&	show&	american&	zzz\_arthur\_andersen&	federal&	peace\\
stock&	group&	games&	public&	information&	chief&	patient&	israelis\\
money&	political&	point&	need&	look&	executive&	system&	israeli\\
business&	zzz\_united\_states&	run&	part&	site&	analyst&	decision&	military\\
government&	zzz\_u\_s&	coach&	family&	zzz\_new\_york&	executives&	bill&	zzz\_palestinian\\
official&	administration&	home&	help&	question&	lay&	member&	zzz\_west\_bank\\
billion&	leader&	win&	job&	number&	investor&	lawyer&	war\\
president&	attack&	won&	teacher&	called&	energy&	doctor&	security\\
plan&	zzz\_white\_house&	night&	country&	find&	investment&	cost&	violence\\
high&	tax&	left&	problem&	found&	employees&	care&	killed\\
fund&	zzz\_washington&	guy&	parent&	ago&	accounting&	health&	talk\\
	\midrulec
	\llap{\smash{\rotatebox[origin=c]{90}{%
      \hspace*{-14\normalbaselineskip}
      \textsc{SpanSPCA}}}%
      \hspace*{2.3em}}
percent&	team&	official&	zzz\_al\_gore&	cup&	show&	public&	night\\
company&	game&	zzz\_bush&	zzz\_george\_bush&	minutes&	com&	member&	big\\
million&	season&	zzz\_united\_states&	campaign&	add&	part&	system&	set\\
companies&	player&	attack&	election&	tablespoon&	look&	case&	film\\
market&	play&	zzz\_u\_s&	political&	teaspoon&	need&	number&	find\\
stock&	games&	palestinian&	vote&	oil&	book&	question&	room\\
business&	point&	military&	republican&	pepper&	women&	job&	place\\
money&	run&	leader&	voter&	water&	family&	told&	friend\\
billion&	right&	zzz\_american&	democratic&	large&	called&	put&	took\\
plan&	win&	war&	school&	sugar&	children&	zzz\_washington&	start\\
government&	coach&	zzz\_israel&	presidential&	serving&	help&	found&	car\\
president&	home&	country&	zzz\_white\_house&	butter&	ago&	information&	feel\\
high&	won&	administration&	law&	chopped&	zzz\_new\_york&	federal&	half\\
cost&	left&	terrorist&	zzz\_republican&	hour&	program&	student&	guy\\
group&	hit&	american&	tax&	pan&	problem&	court&	early\\
	\midrulec
	\llap{\smash{\rotatebox[origin=c]{90}{%
      \hspace*{-14\normalbaselineskip}
      \textsc{SPCABiPart}}}%
      \hspace*{2.3em}}
company&	show&	cup&	team&	percent&	zzz\_al\_gore&	official&	school\\
companies&	home&	minutes&	game&	million&	zzz\_george\_bush&	zzz\_bush&	student\\
stock&	run&	add&	season&	money&	campaign&	government&	children\\
market&	com&	tablespoon&	player&	plan&	right&	president&	women\\
billion&	high&	oil&	play&	business&	election&	zzz\_united\_states&	book\\
zzz\_enron&	need&	teaspoon&	games&	tax&	political&	zzz\_u\_s&	family\\
firm&	look&	pepper&	coach&	cost&	point&	group&	called\\
analyst&	part&	water&	guy&	cut&	leader&	attack&	hour\\
industry&	night&	large&	yard&	job&	zzz\_washington&	zzz\_american&	friend\\
fund&	zzz\_new\_york&	sugar&	hit&	pay&	administration&	country&	found\\
investor&	help&	serving&	played&	deal&	question&	military&	find\\
sales&	left&	butter&	playing&	quarter&	member&	american&	set\\
customer&	put&	chopped&	ball&	chief&	won&	war&	room\\
investment&	ago&	fat&	fan&	executive&	win&	law&	film\\
economy&	big&	food&	shot&	financial&	told&	public&	small\\
	\bottomrulec
   \end{tabular}%
   } 
   \vspace{1em}
   
   \rowcolors{1}{white}{white}
   \begin{tabular}{lc}
	  \toprulec
 					 & Total Cum. Variance \\
	  \midrulec
        \textsc{TPower} &  $48.140645$\\
        \textsc{SpanSPCA} & $48.767864$ \\
        \textsc{SPCABiPart} & $\mathbf{51.873063}$ \\
        \bottomrulec
   \end{tabular}
   \caption{
  \textsc{ BagOfWords:NyTimes} dataset~\cite{Lichman:2013}.
   We run various SPCA algorithms for $k=8$ components (topics)
   and cardinality ${s=15}$ per component. 
   The table lists the words corresponding to each component (words corresponding to higher magnitude entries appear higher in the topic).
   Our algorithm was configured to use a rank-$4$ approximation of the input data.
   }
   \label{bow:nytimes:table-topics-3algos-k8-s15}
\end{table}   

 \renewcommand{\arraystretch}{1.1}
\begin{table}[tbph!]
	\setcounter{magicrownumbers}{0}
	\fontsize{7.0}{8}\selectfont
	\setlength{\tabcolsep}{0.85\tabcolsep}
	\centering
	\rowcolors{2}{white}{mustard!20!white}
	\noindent\makebox[\textwidth]{%
	\begin{tabular}{%
		>{\hspace{-\tabcolsep plus .5em}{\makebox[1.3em][r]{\rownumber\space}}\hspace{.3em}}l<{}%
		>{}l<{}%
		>{}l<{}%
		>{}l<{}%
		>{}l<{}%
		>{}l<{}%
		>{}l<{}%
		>{}l<{\hspace{-\tabcolsep}}%
	}
	\toprulec
	Topic $1$ &
	Topic $2$ &
	Topic $3$ &
	Topic $4$ &
	Topic $5$ &
	Topic $6$ &
	Topic $7$ &
	Topic $8$ \\
	\midrulec
	\midrulec
	\llap{\smash{\rotatebox[origin=c]{90}{%
      \hspace*{-18\normalbaselineskip}
      \textsc{TPower}}}%
      \hspace*{2.3em}}
percent&	zzz\_bush&	team&	school&	com&	zzz\_enron&	drug&	palestinian\\
company&	zzz\_al\_gore&	game&	student&	women&	court&	patient&	zzz\_israel\\
million&	zzz\_george\_bush&	season&	program&	book&	case&	doctor&	zzz\_israeli\\
companies&	campaign&	player&	children&	web&	firm&	cell&	zzz\_yasser\_arafat\\
market&	zzz\_united\_states&	play&	show&	site&	federal&	care&	peace\\
stock&	zzz\_u\_s&	games&	public&	information&	lawyer&	disease&	israelis\\
government&	political&	point&	part&	zzz\_new\_york&	deal&	health&	israeli\\
official&	attack&	run&	family&	www&	decision&	medical&	zzz\_palestinian\\
money&	zzz\_american&	home&	system&	hour&	chief&	test&	zzz\_west\_bank\\
business&	american&	coach&	help&	find&	power&	hospital&	security\\
president&	administration&	win&	problem&	mail&	industry&	research&	violence\\
billion&	leader&	won&	law&	found&	executive&	cancer&	killed\\
plan&	country&	left&	job&	put&	according&	treatment&	talk\\
group&	election&	night&	called&	set&	financial&	study&	meeting\\
high&	zzz\_washington&	hit&	look&	room&	office&	death&	soldier\\
right&	military&	guy&	member&	big&	analyst&	human&	minister\\
fund&	zzz\_white\_house&	yard&	question&	told&	executives&	heart&	zzz\_sharon\\
need&	war&	played&	ago&	friend&	zzz\_arthur\_andersen&	blood&	fire\\
cost&	tax&	start&	teacher&	director&	employees&	trial&	zzz\_ariel\_sharon\\
number&	nation&	playing&	parent&	place&	investor&	benefit&	zzz\_arab\\
	\midrulec
	\llap{\smash{\rotatebox[origin=c]{90}{%
      \hspace*{-18\normalbaselineskip}
      \textsc{SpanSPCA}}}%
      \hspace*{2.3em}}
percent&	team&	zzz\_al\_gore&	attack&	school&	cup&	com&	drug\\
company&	game&	zzz\_bush&	zzz\_united\_states&	student&	minutes&	web&	patient\\
million&	season&	zzz\_george\_bush&	zzz\_u\_s&	children&	add&	site&	cell\\
companies&	player&	campaign&	palestinian&	program&	tablespoon&	information&	doctor\\
market&	play&	election&	military&	family&	oil&	computer&	disease\\
stock&	games&	political&	zzz\_american&	women&	teaspoon&	find&	care\\
business&	point&	tax&	zzz\_israel&	show&	pepper&	big&	health\\
money&	run&	republican&	war&	help&	water&	zzz\_new\_york&	test\\
billion&	win&	zzz\_white\_house&	country&	told&	large&	www&	research\\
government&	home&	vote&	terrorist&	parent&	sugar&	mail&	human\\
president&	won&	law&	american&	problem&	serving&	set&	medical\\
plan&	coach&	administration&	zzz\_taliban&	book&	butter&	put&	study\\
high&	left&	democratic&	zzz\_afghanistan&	job&	chopped&	director&	death\\
group&	night&	voter&	security&	found&	hour&	industry&	cancer\\
official&	hit&	leader&	zzz\_israeli&	friend&	pan&	room&	hospital\\
need&	guy&	public&	nation&	ago&	fat&	small&	treatment\\
right&	yard&	zzz\_republican&	member&	question&	bowl&	car&	scientist\\
part&	played&	presidential&	support&	teacher&	gram&	zzz\_internet&	according\\
cost&	look&	federal&	called&	case&	food&	place&	blood\\
system&	start&	zzz\_washington&	forces&	number&	medium&	film&	heart\\
	\midrulec
	\llap{\smash{\rotatebox[origin=c]{90}{%
      \hspace*{-18\normalbaselineskip}
      \textsc{SPCABiPart}}}%
      \hspace*{2.3em}}
palestinian&	percent&	zzz\_al\_gore&	cup&	school&	team&	company&	official\\
zzz\_israel&	million&	zzz\_bush&	minutes&	right&	game&	companies&	government\\
zzz\_israeli&	money&	zzz\_george\_bush&	add&	group&	season&	market&	president\\
zzz\_yasser\_arafat&	billion&	campaign&	tablespoon&	show&	player&	stock&	zzz\_united\_states\\
peace&	business&	election&	oil&	home&	play&	zzz\_enron&	zzz\_u\_s\\
war&	fund&	political&	teaspoon&	high&	games&	analyst&	attack\\
terrorist&	tax&	zzz\_white\_house&	pepper&	program&	point&	firm&	zzz\_american\\
zzz\_taliban&	cost&	administration&	water&	need&	run&	industry&	country\\
zzz\_afghanistan&	cut&	republican&	hour&	part&	coach&	investor&	law\\
forces&	job&	leader&	large&	com&	win&	sales&	plan\\
bin&	pay&	vote&	sugar&	american&	won&	customer&	public\\
troop&	economy&	democratic&	serving&	look&	left&	price&	zzz\_washington\\
laden&	deal&	presidential&	butter&	help&	night&	investment&	member\\
student&	big&	zzz\_clinton&	chopped&	problem&	hit&	quarter&	system\\
zzz\_pakistan&	chief&	support&	pan&	called&	guy&	executives&	nation\\
product&	executive&	zzz\_congress&	fat&	zzz\_new\_york&	yard&	consumer&	case\\
zzz\_internet&	financial&	military&	bowl&	number&	played&	technology&	federal\\
profit&	start&	policy&	gram&	question&	ball&	share&	information\\
earning&	record&	court&	food&	ago&	playing&	prices&	power\\
shares&	manager&	security&	league&	told&	lead&	growth&	effort\\
	\bottomrulec
   \end{tabular}%
   } 
   \vspace{1em}
   
   \rowcolors{1}{white}{white}
   \begin{tabular}{lc}
	  \toprulec
 					 & Total Cum. Variance \\
	  \midrulec
        \textsc{TPower} &  $50.7686$\\
        \textsc{SpanSPCA} & $52.8117$ \\
        \textsc{SPCABiPart} & $\mathbf{54.8906}$ \\
        \bottomrulec
   \end{tabular}
   \caption{
  \textsc{ BagOfWords:NyTimes} dataset~\cite{Lichman:2013}.
   We run various SPCA algorithms for $k=8$ components (topics)
   and cardinality ${s=20}$ per component. 
   The table lists the words corresponding to each component (words corresponding to higher magnitude entries appear higher in the topic).
   Our algorithm was configured to use a rank-$4$ approximation of the input data.
   }
   \label{bow:nytimes:table-topics-3algos-k8-s20}
\end{table}   

%

\end{document}